\documentclass{article}

\usepackage[normalem]{ulem}
\usepackage{times}
\usepackage{graphicx} 
\usepackage{amsmath}
\usepackage{amssymb}
\usepackage{amsthm}
\usepackage{subfigure}
\usepackage{fullpage}
\usepackage{authblk}
\usepackage{color}
\usepackage[round]{natbib}
\usepackage{verbatim}
\usepackage{algorithm}
\usepackage{algorithmic}
\usepackage{mathtools}

\DeclarePairedDelimiter\floor{\lfloor}{\rfloor}
\usepackage{bbm}
\newtheorem{theorem}{Theorem}[section]
\newtheorem{lemma}[theorem]{Lemma}
\newtheorem{proposition}[theorem]{Proposition}

\theoremstyle{definition}

\theoremstyle{remark}
\newtheorem{remark}{Remark}

\newcommand\numberthis{\addtocounter{equation}{1}\tag{\theequation}}
\usepackage{hyperref}

\newcommand{\grad}{\nabla}
\newcommand{\E}{\mathbb E}
\newcommand{\R}{\mathbb R}
\renewcommand{\P}{\mathbb P}
\newtheorem{assump}{}

\newcommand{\cU}{\mathcal U}
\DeclareMathOperator{\prox}{prox}
\DeclareMathOperator{\argmin}{argmin}

\begin{document}

%

%

\title{SGD with Variance Reduction beyond Empirical Risk Minimization}

\author[1]{Massil Achab\footnote{massil.achab@cmap.polytechnique.fr}}
\author[2]{Agathe Guilloux\footnote{agathe.guilloux@upmc.fr}}
\author[1]{St\'ephane Gaiffas\footnote{stephane.gaiffas@cmap.polytechnique.fr}}
\author[1]{Emmanuel Bacry\footnote{emmanuel.bacry@polytechnique.fr}}
\affil[1]{Centre de Math\'ematiques Appliqu\'ees, CNRS, Ecole Polytechnique, UMR 7641, 91128 Palaiseau, France}
\affil[2]{Laboratoire de Statistique Th\'eorique et Appliqu\'ee, Universit\'e Pierre et Marie Curie, 4 place Jussieu, 75005 Paris, France}

\renewcommand\Affilfont{\itshape\small}

\maketitle

\begin{abstract}
We introduce a doubly stochastic proximal gradient algorithm for optimizing a finite average of smooth convex functions, whose gradients depend on numerically expensive expectations.
Indeed, the effectiveness of SGD-like algorithms relies on the assumption that the computation of a subfunction's gradient is cheap compared to the computation of the total function's gradient. This is true in the Empirical Risk Minimization (ERM) setting, but can be false when each subfunction depends on a sequence of examples.
Our main motivation is the acceleration of the optimization of the regularized Cox partial-likelihood (the core model in survival analysis), but other settings can be considered as well.

The proposed algorithm is doubly stochastic in the sense that gradient steps are done using stochastic gradient descent (SGD) with variance reduction, and the inner expectations are approximated by a Monte-Carlo Markov-Chain (MCMC) algorithm.
We derive conditions on the MCMC number of iterations guaranteeing convergence, and obtain a linear rate of convergence under strong convexity and a sublinear rate without this assumption.

We illustrate the fact that our algorithm improves the state-of-the-art solver for regularized Cox partial-likelihood on several datasets from survival analysis.

\end{abstract}

\smallskip
\noindent \textbf{Keywords.} Convex Optimization, Stochastic Gradient Descent, Monte Carlo Markov Chain, Survival Analysis, Conditional Random Fields

\section{Introduction}
\label{sec:intro}

During the past decade, advances in biomedical technology have brought high dimensional data to biostatistics and survival analysis in particular. Today's challenge for survival analysis lays in the analysis of massively high dimensional (numerous covariates) and large-scale (large number of observations) data, see in particular~\cite{murdoch2013inevitable}. Areas of application outside of biostatistics, such as economics (see~\cite{einav2014economics}), or actuarial sciences~(see \cite{richards2012handbook}) are also concerned.

One of the core models of survival analysis is the Cox model (see~\cite{CoxModel}) for which we propose, in the present paper, a novel scalable optimization algorithm tuned to handle massively high dimensional and large-scale data. Survival data $(y_i, x_i, \delta_i)_{i=1}^{n_\text{pat}}$ contains, for each individual $i=1, \ldots, n_\text{pat}$, a features vector $x_i \in \R^d$, an observed time $y_i \in \R_+$, which is a failure time if $\delta_i = 1$ or a right-censoring time if $\delta_i = 0$.
If $D = \{ i : \delta_i = 1 \}$ is the set of patients for which a failure time is observed, if $n = |D|$ is the total number of failure times, and if $R_i = \{ j : y_{j} \ge y_i \}$ is the index of individuals still at risk at time $y_i$, the negative Cox partial log-likelihood writes
\begin{equation}
  \label{eq:cox-partial}
  - \ell(\theta) = \frac{1}{n} \sum_{i \in D} \Big[ - x_i^\top \theta + \log \Big( \sum_{j \in R_i} \exp ( x_{j}^\top \theta ) \Big)  \Big]
\end{equation}
for parameters $\theta \in \R^d$. This model can be regarded as a regression of the $n$ failure times, using information from the $n_{\text{pat}}$ patients that took part to the study.
With high-dimensional data, a regularization term is added to the partial likelihood to automatically favor sparsity in the estimates, see~\cite{CoxLasso} and~\cite{coxnet} for a presentation of Lasso and elastic-net penalizations, see also the review paper by~\cite{witten2009survival} for an exhaustive presentation. Several algorithms for the Cox model have been proposed to solve the regularized optimization problem at hand, see~\cite{park2007l1, sohn2009gradient,goeman2010l1} among others. These implementations use Newton-Raphson iterations, i.e. large matrices inversions, and can therefore not handle large-scale data.
Cyclical coordinate descent algorithms have since been proposed and successfully implemented in \texttt{R} packages \texttt{coxnet} and \texttt{fastcox}, see~\cite{coxnet,cocktail}.
More recently \cite{large-surv-analysis} adapted the column relaxation with logistic loss algorithm of~\cite{zhang2000value} to the Cox model.
The fact that all these algorithms are of cyclic coordinate descent type solve the problem, supported by Newton-Raphson type algorithms, of large matrices inversions.

Yet another computationnally costly problem, specific to the Cox model, has not been fully addressed: the presence of cumulative sums (over indices ${j \in R_i}$) in the Cox partial likelihood.
This problem was noticed in~\cite{large-surv-analysis}, where a numerical workaround exploiting sparsity is proposed to reduce the computational cost.
The cumulative sum prevents from successfully applying stochastic gradient algorithms, which are however known for their efficiency to handle large scale generalized linear models: see for instance SAG by~\cite{SAG}, SAGA by~\cite{SAGA}, Prox-SVRG by~\cite{SVRG} and SDCA by~\cite{SDCA1} that propose very efficient stochastic gradient algorithms with constant step-size (hence achieving linear rates),
see also Catalyst by~\cite{lin2015universal} that introduces a generic scheme to accelerate and analyze the convergence of those algorithms.

Such recent stochastic gradient algorithms have shown that it is possible to improve upon proximal full gradient algorithms for the minimization of convex problems of the form
\begin{equation}
  \label{eq:min_F}
  \min_{\theta \in \mathbb{R}^d} F(\theta) = f(\theta) + h(\theta) \text{ with } f(\theta) = \frac{1}{n} \sum_{i=1}^n f_i(\theta),
\end{equation}
where the functions $f_i$ are gradient-Lipschitz and $h$ is prox-capable. These algorithms take advantage of the finite sum structure of $f$, by using some form of variance-reduced stochastic gradient descent.
It leads to algorithms with a much smaller iteration complexity, as compared to proximal full gradient approach (FG), while preserving (or even improving) the linear convergence rate of FG in the strongly convex case.
However, such algorithms are relevant when gradients $\grad f_i$ have a numerical complexity much smaller than $\grad f$, such as for linear classification or regression problems, where $\grad f_i$ depends on a single inner product $x_i^\top \theta$ between features $x_i$ and parameters $\theta$.

In this paper, motivated by the important example of the Cox partial likelihood~\eqref{eq:cox-partial}, we consider the case where gradients $\grad f_i$ can have a complexity comparable to the one of $\grad f$.
More precisely, we assume that they can be expressed as expectations, under a probability measure $\pi_\theta^i$, of random variables
$G_i (\theta)$, i.e.,
\begin{equation}
  \label{eq:grad_f_i}
  \nabla f_i(\theta) = \mathbb{E}^{G_i (\theta) \sim \pi_\theta^i} [G_i (\theta)].
\end{equation}
This paper proposes a new doubly stochastic proximal gradient descent algorithm (2SVRG), that leads to a low iteration complexity, while preserving linear convergence under suitable conditions for problems of the form~\eqref{eq:min_F}~+~\eqref{eq:grad_f_i}.

Our main motivation for considering this problem is to accelerate the training-time of the the penalized Cox partial-likelihood.
The function
$-\ell(\theta)$ is convex (as a sum of linear and log-sum-exp functions, see Chapter~3 of~\cite{boyd_book}, and fits in the setting~\eqref{eq:min_F}~+~\eqref{eq:grad_f_i}.
Indeed, fix $i \in D$ and introduce
\begin{equation*}
  f_i(\theta) = - x_i^\top \theta + \log \Big( \sum_{j \in R_i} \exp ( x_{j}^\top \theta ) \Big),
\end{equation*}
so that
\begin{equation*}
  \grad f_i(\theta) = -x_i + \sum_{j \in R_i} x_j \pi_\theta^i(j)
\end{equation*}
where
\begin{equation*}
  \pi_\theta^i(j) = \frac{\exp (x_{j}^\top \theta)}{\sum_{j' \in R_i} \exp ( x_{j'}^\top \theta )},~~~\forall j \in R_i.
\end{equation*}
This entails that $\grad f_i(\theta)$ satisfies~\eqref{eq:grad_f_i} with $G_i(\theta)$ a random variable valued in $\{ -x_i + x_j : j \in R_i \}$ and such that
\begin{equation*}
  \P(G_i(\theta) = -x_i + x_j) = \pi_\theta^i(j)
\end{equation*}
for $j \in R_i$. Note that the numerical complexity of $\grad f_i$ can be comparable to the one of $\grad f$, when $y_i$ is close to $\min_i y_i$ (recalling that $R_i = \{ j : y_{j} \ge y_i \}$). Note also that a computational trick allows to compute $\grad f(\theta)$ with a complexity $O(nd)$. 
This makes this setting quite different from the usual case of empirical risk minimization (linear regression, logistic regression, etc.), where all the gradients $\grad f_i$ share the same low numerical cost.

\section{Comparison with previous work}
\label{sec:comparison_with_previous_work}

\paragraph{SGD techniques.}
\label{par:sgd_techniques}

Recent proximal stochastic gradient descent
algorithms by~\cite{SAGA},~\cite{SVRG},~\cite{SDCA1} and~\cite{SAG} build on the idea of~\cite{robbins1951} and~\cite{kiefer1952}.
Such algorithms are designed to tackle large-scale optimization problems ($n$ is large), where it is
assumed implicitly that the $\grad f_i$ (smooth gradients) have a low
computational cost compared to $\grad f$, and where $h$ is eventually non-differentiable and is dealt with using a backward or projection step using its
proximal operator.

The principle of SGD is, at each iteration $t$, to sample uniformly at random an index $i \sim \cU[n]$, and to apply an update step of the form
\begin{equation*}
  \theta^{t+1} \gets \theta^t - \gamma_t \grad f_i(\theta^t).
\end{equation*}
This step is based on an unbiased but very noisy estimate of the full gradient $\grad f$, so the choice of the step size $\gamma_t$ is crucial since it has to be decaying to curb the variance introduced by random sampling (excepted for averaged SGD in some particular cases, see~\cite{bach2013non}). This tends to slow down convergence to a minimum $\theta^\star \in \argmin_{\theta \in \R^d} f(\theta)$.
Gradually reducing the variance of $\grad f_i$ for $i \sim \cU[n]$ as an approximation of $\grad f$ allows to use larger -- even constant -- step sizes and to obtain faster convergence rates. This is the underlying idea of two recent methods - SAGA and SVRG respectively introduced in ~\cite{SAGA},~\cite{SVRG} - that use updates of the form
\begin{equation*}
  w^{t+1} \gets \theta^t - \gamma \Big( \nabla f_i (\theta^t) - \nabla f_i (\tilde \theta) + \frac{1}{n} \sum_{j=1}^n \nabla f_j (\tilde \theta) \Big),
\end{equation*}
and $\theta^{t+1} \gets \prox_{\gamma h}(w^{t+1})$. In~\cite{SVRG},
$\tilde \theta$ is fully updated after a certain number of iterations, called \emph{phases}, whereas in~\cite{SAGA}, $\tilde \theta$  is partially updated after each iteration.
Both methods use stochastic gradient descent steps, with variance reduction obtained via the centered control variable $-\nabla f_i (\tilde \theta) + \frac{1}{n} \sum_{j=1}^n \nabla f_j (\tilde \theta)$, and achieve linear convergence when $F$ is strongly-convex, namely $\E F(\theta^k) - \min_{ \in \R^d} F(x) =~O(\rho^k)$ with $\rho < 1$, which make these algorithms state-of-the-art for many convex optimization problems.
Some variants of SVRG~\cite{SVRG} also approximate the full gradient $\frac{1}{n} \sum_{j=1}^n \nabla f_j (\tilde \theta)$ using mini-batchs to decrease the computing time of each phase, see~\cite{lei2016less, harikandeh2015stopwasting}.

\paragraph{Numerically hard gradients.}

A very different, nevertheless classical, ``trick'' to reduce the complexity  of the gradient computation, is to express it, whenever the statistical problem allows it, as the expectation, with respect to a non-uniform distribution $\pi_\theta$, of a random variable $G(\theta)$, i.e.,  $\nabla f (\theta) = \E^{G(\theta) \sim \pi_\theta} [ G(\theta)]$.
Optimization problems with such a gradient have generated an extensive literature from the first works by~\cite{robbins1951}, and~\cite{kiefer1952}.
Some algorithms are designed to construct stochastic approximations of the sub-gradient of $f + h$, see~\cite{nemirovski, juditsky, Lan10, Duchi}.
Others are based on proximal operators to better exploit the smoothness of $f$ and the properties of $h$, see~\cite{hu_pan, lin_xiao, moulines}.
In this paper, we shall focus on the second kind of algorithms. Indeed, our approach is closer to the one developed in~\cite{moulines}, though, as opposed to ours, the algorithm developed in this latter work is based on proximal full gradient algorithms (not doubly stochastic as ours) and does not guarantee a linear convergence.

\paragraph{Contrastive divergence.}
The idea to approximate the gradient using MCMC already appeared in the litterature of Undirected Graphical Models under the name of Contrastive Divergence, see~\cite{Murphy:2012:MLP:2380985, hinton2002training, carreira2005contrastive}.
Indeeed, for this class of model, the gradient of the log-likelihood $\nabla f (\theta)$ can be written as the difference of two expectations: one - tractable - with respect to the data discrete distribution $\textbf{X}$, the other - intractable - with respect to the model-dependent distribution $p(\cdot,\theta)$.
The idea of Contrastive Divergence relies in the approximation of the intractable expectation using MCMC, with few iterations of the chain.
However, in the framework of Cox model, and also Conditional Random Fields (see Section \ref{sec:conclusion} below), this is the gradient $\grad f_i (\theta)$ that writes as an time-consuming expectation, see Equation~\ref{eq:grad_f_i}.

\paragraph{Our setting.}

The setting of our paper is original in the sense that it combines both previous settings, namely stochastic gradient descent and MCMC. As in the stochastic gradient setting, the gradient can be expressed as the sum of $n$ components, where
$n$ can be very large.
However, since these components are time-consuming to compute directly, following the expectation based gradient computation setting, they are expressed as averaged values of some random variables.
More precisely, the gradient $\grad f_i(\theta)$ is replaced by
an approximation $\widehat \nabla f_i (\theta)$ obtained by an MCMC algorithm.
Our algorithm is, to the best of our knowledge, the first one
to propose a combination of two stochastic approximations in this way, hence the name \emph{doubly stochastic}, which allow to deal with both, eventual large values for $n$ and the inner complexity of each gradient $\grad f_i$ computation.

The idea to mix SGD and MCMC has also been raised recently in the very different setting of \emph{implicit} stochastic gradient descent, see~\cite{implicit_sgd}.
Note also that in our approach we make two stochastic approximations to the gradient using random training points, while the doubly stochastic approach from~\cite{double_sto} performs two stochastic approximations to the gradient using random training points and random features for kernel methods.

\section{A doubly stochastic proximal gradient descent algorithm}
\label{sec:doubly_stochastic_proximal_gradient_descent}

Our algorithm 2SVRG is built upon the algorithm SVRG via an approximation function \textsc{ApproxMCMC}.
We first present the meta-algorithm without specifying the approximation function, and then provide two examples for \textsc{ApproxMCMC}.

\subsection{2SVRG: a meta-algorithm}

Following the ideas presented in the previous section, we design a
\emph{doubly stochastic proximal gradient descent algorithm} (2SVRG), by combining a variance reduction technique for SGD given by Prox-SVRG~\cite{SVRG}, and a Monte-Carlo Markov-Chain algorithm to obtain an approximation of the gradient $\nabla f_j (\theta)$ at each step.
Thus, in the considered setting the full gradient writes
\begin{equation*}
  \grad f(\theta) = \E^{i \sim \cU} [\grad f_i(\theta)] = \E^{i \sim \cU} \; \E^{G_i(\theta) \sim \pi_\theta^i} [G_i(\theta)],
\end{equation*}
where $\cU$ is the uniform distribution on $\{1, \ldots, n\}$, so our algorithm contains two levels of stochastic approximation: uniform sampling of $i$ (the variance-reduced SGD part) for the first expectation, and an approximation of the second expectation w.r.t $\pi_\theta^i$ by means of Monte-Carlo simulation. The 2SVRG algorithm is described in Algorithm~\ref{alg:s2vrg}.

\begin{algorithm}
  \caption{Doubly stochastic proximal gradient descent (2SVRG)}
  \label{alg:s2vrg}
  \small
  \begin{algorithmic}[1]
    \REQUIRE Number of phases $K \geq 1$, phase-length $m \geq 1$, step-size $\gamma > 0$, MCMC number of iterations per phase $(N_k)_{k=1}^K$, starting point $\theta^0 \in \R^d$
    \STATE \textbf{Initialize:} $\tilde{\theta} \gets \theta^0$ and compute $\grad f_i(\tilde \theta)$ for $i=1, \ldots, n$
    \FOR{$k=1$ {\bfseries to} $K$}
    \FOR{$t=0$ {\bfseries to} $m-1$}
    \STATE Pick $i \sim \mathcal{U}[n]$
    \STATE $\widehat \grad f_i (\theta^t) \gets \mbox{\textsc{ApproxMCMC}}(i, \theta^t, N_k)$
    \STATE $d^t = \widehat \grad f_i(\theta^{t}) - \grad f_i(\tilde \theta) + \frac1n \sum_{j=1}^n \grad f_{j} (\tilde \theta)$
    \STATE $\omega^{t+1} \gets \theta^{t} - \gamma d^t$
    \STATE $\theta^{t+1} \gets \prox_{\gamma h}(\omega^{t+1})$
    \ENDFOR
    \STATE Update $\tilde\theta \gets \frac{1}{m} \sum_{t=1}^m \theta^t$, $\theta^0 \gets \tilde \theta$, $\tilde \theta^k \gets \tilde \theta$
    \STATE Compute $\grad f_i(\tilde \theta)$ for $i=1, \ldots, n$
    \ENDFOR%
    \STATE \textbf{Return:} $\tilde \theta^K$%
  \end{algorithmic}
\end{algorithm}
Following Prox-SVRG by~\cite{SVRG},
this algorithm decomposes in \emph{phases}: iterations within a phase apply variance reduced stochastic gradient steps (with a backward proximal step, see lines~7 and~8 in Algorithm~\ref{alg:s2vrg}).
At the end of a phase, a full-gradient is computed (lines~10, 11) and used in the next phase for variance reduction.
Within a phase, each inner iteration samples uniformly at random an index~$i$ (line~4) and obtains an approximation of the gradient $\grad f_i$ at the previous iterate $\theta^t$ by applying $N_k$ iterations of a Monte-Carlo Markov-Chain (MCMC) algorithm.

Intuitively, the sequence $N_k$ should be increasing with the phase number $k$, as we need more and more precision as the iterations goes on (this is confirmed in Section~\ref{sec:theory}).
The important point of our algorithm resides precisely in this aspect: very noisy estimates can be used in the early phases of the algorithm, hence allowing for an overall low complexity as compared to a full gradient approach.

\subsection{ Choice of \textsc{ApproxMCMC} }

We focus now on two implementations of the function \textsc{ApproxMCMC} based on two famous MCMC algorithms: Metropolis-Hastings and Importance Sampling.

\subsubsection{ Independent Metropolis-Hastings }
When the $\pi_\theta^i$ are Gibbs probability measures, as for the previously described Cox partial log-likelihood (but for other models as well, such as Conditional Random Fields, see~\cite{crf1}), one can apply Independent Metropolis-Hastings (IMH), see Algorithm~\ref{IMH} below, to obtain approximations $\widehat \grad f_i$ of the gradients.
In this case the produced chain is geometrically uniformly ergodic, see~\cite{robert2004monte}, and therefore meets the general  assumptions required in our results (see Proposition~\ref{prop:IMH} below).
The IMH algorithm uses a proposal distribution $Q$ which is independent of the current state $j_l$ of the Markov chain.
\begin{algorithm}
\caption{\text{Independent Metropolis-Hastings (IMH) estimator (for the Cox model)}}
\label{IMH}
\begin{algorithmic}
\REQUIRE Proposal distribution $Q=\mathcal{U}\{R_i\}$, starting point $j_0 \in R_i$, stationary distribution $\pi=\pi_{\theta^t}^i$
\FOR{$l=0, \ldots, N_k-1$}%
\STATE $1.\textbf{ Generate: } j' \sim Q$.
\STATE $2.\textbf{ Update: } \alpha = \min \left( \frac{\pi (j') Q(j_l)}{\pi (j_l) Q (j')}, 1 \right) = \min \left( \exp ((x_{j'} - x_{j_l})^\top \theta^t), 1 \right)$.
\STATE $3.\textbf{ Take: }
j_{l+1} =
\begin{cases}
j' & \text{ with probability } \alpha \\
j_l & \text{ otherwise.}
\end{cases}
$
\ENDFOR%
\STATE \textbf{Return:} $-x_i + \frac{1}{N_k} \sum_{l=1}^{N_k} x_{j_l}$
\end{algorithmic}
\end{algorithm}

In the case of the Cox partial log-likelihood, at iteration $t$ of phase $k$ of Algorithm~\ref{alg:s2vrg}, we set $\pi = \pi_{\theta^t}^i$, and $Q$ to be the uniform distribution over the set $R_i$.
We implemented two versions of Algorithm~\ref{alg:s2vrg} with IMH: one with a uniform proposal $Q$, the other one with an adaptative proposal $\widetilde{Q}$.
When we want to approximate $\grad f_i (\theta)$, we can consider the adaptative proposal $\widetilde{Q}=\pi_{\tilde{\theta}}^i$, where $\tilde{\theta}$ is the iterate we have computed at the end of the previous phase, see Line 10 of Algorithm~\ref{alg:s2vrg}.
Since we compute the full gradient only once every phase, the probabilities $\pi_{\tilde{\theta}}^i(j)$ are computed at the same time, which means that the use of an adaptative proposal adds no computational effort.
Morever, the theoretical guarantees given in Section~\ref{sec:theory} make no difference between the two versions aformentionned, but a strong difference is observed in practice.

\subsubsection{ Importance Sampling }

To choice of the adaptative proposal above reduces the variance of the estimator given by \textsc{ApproxMCMC}.
The idea of sampling with $\widetilde{Q} = \pi_{\tilde{\theta}}^i$ can also be used in an Importance Sampling estimator as well.
\begin{align*}
\grad f_i( \theta ) &= \mathbb{E}^{G_i (\theta) \sim \pi_{\theta}^i} \left[ G_i (\theta) \right] = \mathbb{E}^{G_i (\theta) \sim \widetilde{Q}} \left[ G_i (\theta) \frac{\pi_{\theta}^i (G_i(\theta))}{\widetilde{Q} (G_i (\theta))} \right]
\end{align*}
Since the ratio $\pi_{\theta}^i (G_i(\theta)) / \widetilde{Q} (G_i (\theta))$ still contains an expensive term to compute, we can divide the term above with $\mathbb{E}_{\widetilde{Q}} \left[ \pi_{\theta}^i (G_i(\theta)) / \widetilde{Q} (G_i (\theta)) \right]=1$ and approximate the resulting term.
This trick provides an estimator called Normalized Importance Sampling estimator, which writes like this in the case of Cox partial likelihood:
\begin{align*}
  \widehat{J}_N &= \left. \sum_{k=1}^N (x_{j_k} - x_i) \frac{\pi_{\theta}^i (j_k)}{\widetilde{Q}(j_k)} \middle/ \sum_{k=1}^N \frac{\pi_{\theta}^i (j_k)}{\widetilde{Q}(j_k)} \right., &\mbox{ with } j_k \sim \widetilde{Q} \\
  &= -x_i + \sum_{k=1}^N \frac{\exp ( (\theta-\tilde{\theta})^\top x_{j_k} ) }{ \sum_{l=1}^N \exp ( (\theta-\tilde{\theta})^\top x_{j_l} )} x_{j_k}, &\mbox{ with } j_k \sim \widetilde{Q}
\end{align*}

\begin{algorithm}
\caption{\text{Normalized Importance Sampling (NIS) estimator of $\grad f_i (\theta)$ (for the Cox model)}}
\label{AIS}
\begin{algorithmic}
\REQUIRE Proposal distribution $\widetilde{Q} = \pi_{\tilde{\theta}}^i$, stationary distribution $\pi_\theta^i$, $V=0 \in \R^d, S=0 \in \R$
\FOR{$l=1, \ldots, N_k$}%
\STATE $1.\textbf{ Generate: } j_{l} \sim \widetilde{Q}(\cdot)$.
\STATE $2. \textbf{ Update: } V \gets V + \exp ((\theta - \tilde{\theta})^\top {x_{j_l}}) x_{j_l}$.
\STATE $3. \textbf{ Update: } S \gets S + \exp ((\theta - \tilde{\theta})^\top {x_{j_l}})$.
\ENDFOR%
\STATE \textbf{Return:} $-x_i + V / S$
\end{algorithmic}
\end{algorithm}

Section~\ref{sec:theory} below gives theoretical guarantees for Algorithm~\ref{alg:s2vrg}: linear convergence under strong-convexity of~$F$ is given in Theorem~\ref{thm:1}, and a convergence without strong convexity is given in Theorem~\ref{thm:2}.
This improves the proximal stochastic gradient method of~\cite{moulines}, where the best case rate is $O(1 / k^2)$ using Fista (see~\cite{FISTA}) acceleration scheme.
Numerical illustrations are given in Section~\ref{sec:expes}, where a fair comparison between several state-of-the-art algorithms is proposed.

\section{Theoretical guarantees}
\label{sec:theory}

\paragraph{Definitions.} All the functions $f_i$ and $h$ are proper convex lower-semicontinuous on $\R^d$. The norm $\|\cdot\|$ stands for the Euclidean norm on $\R^d$. A function $f : \R^d \rightarrow \R$ is $L$-smooth if it is differentiable and if its gradient is $L$-Lipschitz, namely if $\| \nabla f(x) - \nabla f(y) \| \le L \|x - y \|$ for all $x, y \in \R^d$. A function $f : \mathbb{R}^d \rightarrow \mathbb{R}$ is $\mu$-strongly convex if $f(x + y) \ge f(x) + \nabla f (x)^\top y + \frac{\mu}{2} \| y \|^2$ for all $x, y \in \R^d$ i.e. if $f - \frac{\mu}{2} \| \cdot \|^2$ is convex. The proximal operator of $h : \R^d \rightarrow \mathbb{R}$ is uniquely defined by $\mbox{prox}_h (x) = \argmin_{y \in \mathbb{R}^d} \{ h(y) + \frac{1}{2} \|x - y\|^2 \}$.

\paragraph{Notations.} We denote by $i_t$ the index randomly picked at the $t^{th}$ iteration, see line~4 in Algorithm~\ref{alg:s2vrg}. We introduce the error of the MCMC approximation $\eta^{t} = \widehat \grad f_{i_t} (\theta^{t-1}) - \nabla f_{i_t} (\theta^{t-1})$ and the filtration $\mathcal{F}_{t} = \sigma(\theta^0, i_1, \theta^1, \ldots, i_{t}, \theta^{t})$. In order to analyze the descent steps, we need different expectations: $\mathbb{E}_t$ the expectation w.r.t the distribution of the pair $(i_t, \widehat \grad f_{i_t} (\theta^{t-1}))$ conditioned on $\mathcal{F}_{t-1}$, and $\mathbb{E}$ the expectation w.r.t all the random iterates $(i_t, \theta^t)$ of the algorithm. We also denote $\theta^* = \argmin_{\theta \in \mathbb{R}^d} F(\theta)$.

\paragraph{Assumptions.}

\begin{assump}
\label{assump:1}
We consider $F = f + h$ where $f = \frac{1}{n} \sum_{i=1}^n f_i$, with each $f_i$ being convex and $L_i$-smooth, $L_i > 0$, and $h$ a lower semi-continuous and closed convex function. We denote $L = \max_{1 \le i \le n} L_i$. We assume that there exists $B > 0$ such that the iterates $\theta^t$ satisfy $\sup_{t \ge 0} \|\theta^t - \theta^*\| \le B$.
\end{assump}
\begin{assump}
\label{assump:2}
We assume that the bias and the expected squared error of the Monte Carlo estimation can be bounded in the following way:
\begin{align}
\label{eq:C1_C2}
\|\mathbb{E}_t \eta^t\| \le \frac{C_1}{N_k} \mbox{ and } \mathbb{E}_t \| \eta^{t} \|^2  \le \frac{C_2}{N_k}
\end{align}
for the iterations $t$ belonging to the $k$-th phase, where $N_k$ is the number of iterations of the Markov chain used for the computation of $\widehat \grad f_{i_t}(\theta^t)$ during phase $k$ (see line~5 of Algorithm~\ref{alg:s2vrg}), and where $C_1$ and $C_2$ are positive constants.
\end{assump}

Let us point out that Proposition~\ref{prop:IMH} below gives a sufficient condition for~\ref{assump:2} to hold.

\paragraph{Theorems.} The theorems below provide upper bounds on the distance to the minimum in the strongly convex case, see Theorem~\ref{thm:1} and in the convex case, see Theorem~\ref{thm:2}.

\begin{theorem}
\label{thm:1}
Suppose that $F = f + h$ is $\mu$-strongly convex. Consider
Algorithm~\ref{alg:s2vrg}, with a phase length $m$ and a step-size $\gamma \in (0, \frac{1}{16 L})$ satisfying
\begin{equation}
\rho = \frac{1}{m \gamma \mu (1 - 8 L \gamma)} + \frac{8 L \gamma (1 + 1 / m)}{1 - 8 L \gamma} < 1.
\end{equation}
Then, under~\ref{assump:1} and~\ref{assump:2}, we have\textup:
\begin{equation}
\mathbb{E} [F(\tilde{\theta}^{K})] - F(\theta^*) \le \rho^K \Big( F(\theta^{0}) - F(\theta^*) + \sum_{l = 1}^{K} \frac{D}{\rho^l N_l} \Big),
\end{equation}
where $D = \frac{3 \gamma C_2 + B C_1}{1 - 8 L \gamma}$.
\end{theorem}

In Theorem~\ref{thm:1}, the choice $N_k = k^\alpha \rho^{-k}$ with $\alpha > 1$ gives
\begin{align*}
\mathbb{E} [F(\tilde{\theta}^{K})] - F(\theta^*) \le D' \rho^K
\end{align*}
where $D' = F(\theta^{0}) - F(\theta^*) + D \sum_{k \geq 1} k^{-\alpha}$ and $D > 0$ is a numerical constant. This entails that 2SVRG achieves a \emph{linear rate} under strong convexity.

\begin{remark}
  [An important remark] The number $N_k$ of MCMC iterations is growing quickly with the phase number $k$. So, we use in practice an hybrid version of 2SVRG called \emph{HSVRG}: 2SVRG is used for the first phases (usally 4 or 5~phases in our experiments), and as soon as $N_k$ exceeds $n$, we switch to a mini-batch version of Prox-SVRG (SVRG-MB), see~\cite{SVRG_MB}.
  A precise description of HSVRG is given in Algorithm~\ref{alg:hybrid} from Section~\ref{sec:expes} below.
  Note that overall linear convergence of HSVRG is still guaranteed, since both 2SVRG and SVRG-MB decrease linearly the objective from one phase to the other.
\end{remark}

\begin{theorem}
\label{thm:2}
Consider
Algorithm~\ref{alg:s2vrg}, with a phase length $m$ and a step-size $\gamma \in (0, \frac{1}{8 L (2m+1)})$. Then, under~\ref{assump:1} and~\ref{assump:2}, we have\textup:
\begin{equation}
\mathbb{E} [F(\bar{\theta}^K)] - F(\theta^*) \le \frac{D_1}{K} + \frac{D_2}{K} \sum_{k=1}^{K+1} \frac{1}{N_k},
\end{equation}
where $D_1$ and $D_2$ depend on the constants of the problem, and where $\bar{\theta}^K$ is the average of iterates $\tilde \theta^k$ until phase $K$.
\end{theorem}

In Theorem~\ref{thm:2}, the choice $N_k = k^\alpha$ with $\alpha > 1$ gives
\begin{equation*}
\mathbb{E} [F(\bar{\theta}^K)] - F(\theta^*) \le \frac{D_3}{K}
\end{equation*}
for a constant $D_3 > 0$. This result is an improvement of the Stochastic Proximal Gradient algorithm from~\cite{moulines} since it is not necessary to design a weighted averaged but just a simple average to reach the same convergence rate. Also, it provides a convergence guarantee for the non-strongly convex case, which is not proposed in~\cite{SVRG}.

Theorems~\ref{thm:1} and~\ref{thm:2} show a trade-off between the linear convergence of the variance-reduced stochastic gradient algorithm and the MCMC approximation error. The next proposition proves that Algorithm~\ref{IMH} satisfies~\ref{assump:2} under a general assumption on the proposal and the stationary distribution.

\begin{proposition}
\label{prop:IMH}
Suppose that there exists $M > 0$ such that the proposal $Q$ and the stationary distribution $\pi$ satisfy $\pi (x) \le M Q(x)$, for all $x$ in the support of $\pi$. Then, the error $\eta^t$ obtained by Algorithm~\ref{IMH} satisfies~\ref{assump:2}.
\end{proposition}

\begin{remark}
[Specifics for the Cox partial likelihood]
Note that the assumptions required in Proposition~\ref{prop:IMH} are met for the Cox partial likelihood: in this case, a simple choice is $M = n \max_{x \in \text{supp}(\pi)} \pi (x)$, and the Monte Carlo error $\eta^t$ induced by computing the gradient of $f_i$ at phase $k$ using Algorithm~\ref{IMH} satisfies~\eqref{eq:C1_C2} with
\begin{align*}
  C_1 &= \frac{2}{|R_i|} \max_{j \in R_i} \pi_{\theta^{t-1}}^i (j) \\
  C_2 &= 36 \mathcal{C}_2 C_1^2 (1 + C_1) \max_{j \in R_i} \| x_j \|_2^2,
\end{align*}
where $\mathcal{C}_2$ is the Rosenthal constant of order 2, see~Proposition~12 in~\cite{fort2003}.

\end{remark}

\section{Numerical experiments} 
\label{sec:expes}

We compare several solvers for the minimization of the objective given by an elastic-net penalization of the Cox partial likelihood
\begin{equation*}
  F(\theta) = -\ell(\theta) + \lambda \Big(\alpha \| \theta \|_1 + \frac{1-\alpha}{2} \| \theta \|_2^2 \Big),
\end{equation*}
where we recall that the partial likelihood $\ell$ is defined in Equation~\eqref{eq:cox-partial} and where $\lambda > 0$ and $\alpha \in [0, 1]$ are tuning parameters.

\paragraph{A fair comparison of algorithms.} 

The doubly stochastic nature of the considered algorithms makes it hard to compare them to batch algorithms in terms of iteration number or epoch number (number of full passes over the data), as this is usually done for SGD-based algorithm.
Hence, we proceed by plotting the evolution of $F(\tilde \theta) - F(\theta^*)$ (where $\theta^* \in \argmin_{u \in \R^d} F(u)$ and $\tilde \theta$ is the current iterate of a solver) as a function of the number of inner products between a feature vector $x_i$ and $\theta$, effectively computed by each algorithm, to obtain the current iterate $\tilde \theta$.
This gives a fair way of comparing the effective complexity of all algorithms.

\paragraph{About the baselines specific to the Cox model.}

State-of-the-art algorithms to fit the elastic-net penalized Cox partial likelihood are \texttt{cocktail} by~\cite{cocktail} and \texttt{coxnet}, by~\cite{coxnet}. Both algorithms are combining the ideas of coordinate descent and majoration-minimization. Full convergence results for these algorithms have not yet been established, although \texttt{Cocktail} has a coordinate-wise descent property.

These algorithms however  need a good starting point (near the actual minimizer)  to achieve convergence (this fact is due to a diagonal approximation of the Hessian matrix, see \cite{GAD}, Chapter 8.). They are therefore tuned to provide good path of solutions while varying by small steps the penalization parameter $\lambda$. Indeed in this case, this starting point is naturally set at the minimizer at the previous value of $\lambda$, when minimizing along a path but cannot be guessed outside of a path. We illustrate this fact on Figure~\ref{fig:cocktail_no_conv}, where the convergence of \texttt{Cocktail}  and L-BFGS-B algorithms are compared for two starting points $\theta_0$. 

\begin{figure}
  \centering
  \includegraphics[width=0.48\textwidth]{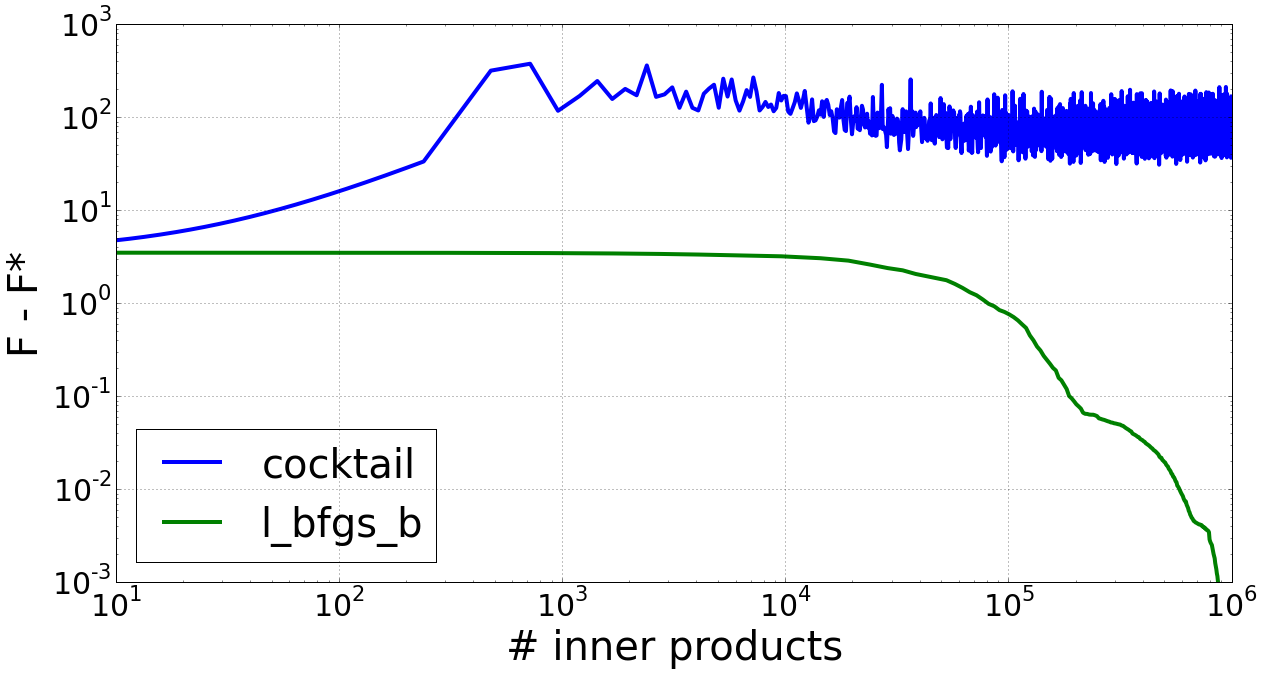}
  \includegraphics[width=0.48\textwidth]{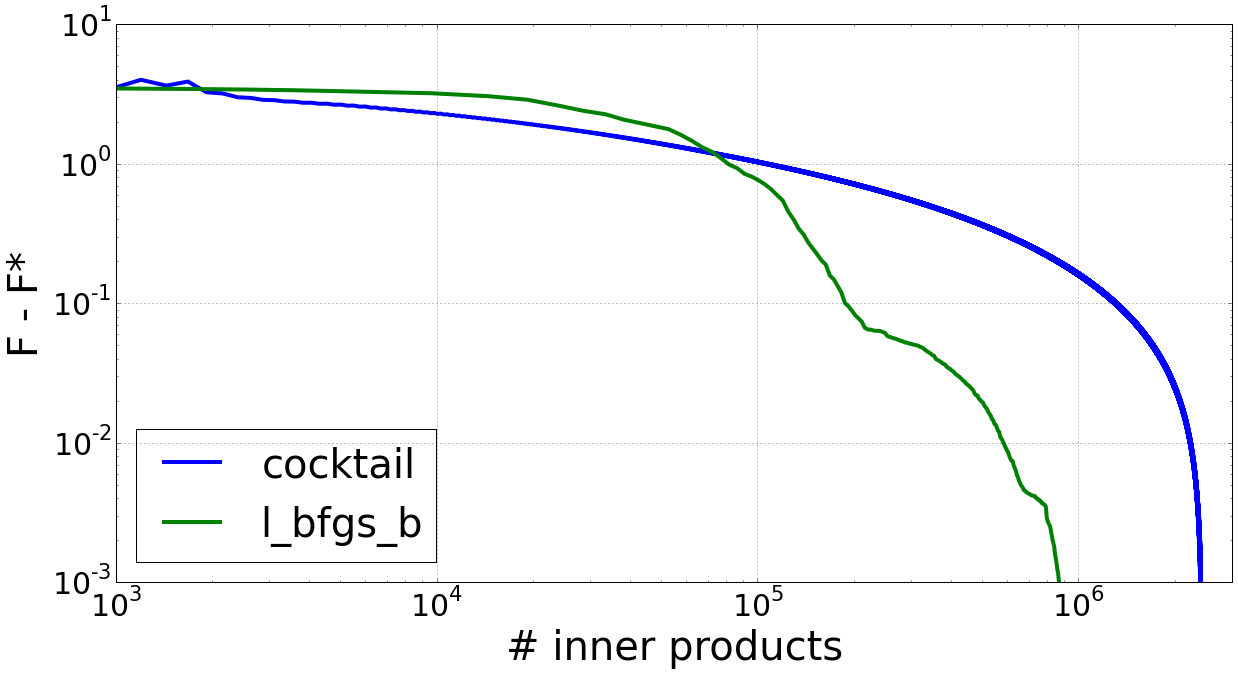}
  \caption{Convergence of Cocktail and L-BFGS-B on Lymphoma dataset. \emph{Top:} the starting point is $\theta^0 = \textbf{0} \in \mathbb{R}^d$. \emph{Bottom:} the starting point is $\theta^0 = \widehat{\theta}^{(l)}$ (solution to the same objective with a slightly larger $\lambda$). This illustrates the fact that Cocktail cannot minimize directly a single objective (with a fixed $\lambda$) and requires to compute the full path of solution to converge.}
  \label{fig:cocktail_no_conv}
\end{figure}

Even when the starting point is set to the previous minimizer (second case in Figure~\ref{fig:cocktail_no_conv}, \texttt{cocktail}'s convergence is slower than the one of L-BFGS-B. As a consequence, we decided that no fair comparison could be conducted with \texttt{cocktail} and \texttt{coxnet} algorithms.

\paragraph{Hybrid SVRG algorithm}
\label{hybrid_algo}

Since $N_k$ exponentially increases, the 2SVRG's complexity is higher than SVRG's original complexity.
However, the algorithm 2SVRG is very efficient during the first phases: we introduce an hybrid solver that begins with 2SVRG and switches to SVRG with mini-batchs (denoted SVRG-MB).
Mini-batching simply consists in replacing single stochastic gradients $\grad f_i$ by an average over a subset $\mathcal B$ of size $n_{\text{mb}}$ uniformly selected at random.
This is useful in our case, since we can use a computational trick (recurrence formula) to compute mini-batched gradients. 
In our experiments, we used $n_\text{mb} = 0.1 n$ or $n_\text{mb} = 0.01 n$, a constant step-size $\gamma$ designed for each dataset, and switched from 2SVRG to SVRG-MB after $K_S=5$ phases. We set $N_k = n^{k / (K_S + 2)}$ so that $N_k$ never exceeds $n$.

\begin{algorithm}
  \caption{Hybrid SVRG (HSVRG)}
  \label{alg:hybrid}
  \small
  \begin{algorithmic}[1]
    \REQUIRE Number of phases before switching $K_S \geq 1$, total number of phases $K \geq K_S$, phase-length $m \geq 1$, step-size $\gamma > 0$, MCMC number of iterations per phase $(N_k)_{k=1}^K$, starting point $\theta^0 \in \R^d$
    \STATE \textbf{Initialize:} $\tilde{\theta} \gets \theta^0$ and compute $\grad f_i(\tilde \theta)$ for $i=1, \ldots, n$
    \FOR{$k=1$ {\bfseries to} $K_S$}
    \FOR{$t=0$ {\bfseries to} $m-1$}
    \STATE Pick $i \sim \mathcal{U}[n]$
    \STATE $\widehat \grad f_i (\theta^t) \gets \mbox{\textsc{ApproxMCMC}}(i, \theta^t, N_k)$
    \STATE $d^t = \widehat \grad f_i(\theta^{t}) - \grad f_i(\tilde \theta) + \frac1n \sum_{j=1}^n \grad f_{j} (\tilde \theta)$
    \STATE $\omega^{t+1} \gets \theta^{t} - \gamma d^t$
    \STATE $\theta^{t+1} \gets \prox_{\gamma h}(\omega^{t+1})$
    \ENDFOR
    \STATE Update $\tilde\theta \gets \frac{1}{m} \sum_{t=1}^m \theta^t$, $\theta^0 \gets \tilde \theta$, $\tilde \theta^k \gets \tilde \theta$
    \STATE Compute $\grad f_i(\tilde \theta)$ for $i=1, \ldots, n$
    \ENDFOR%
    \FOR{$k=K_S+1$ {\bfseries to} $K$}
    \FOR{$t=0$ {\bfseries to} $m_\text{mb}-1=\floor{(m-1) / n_\text{mb}}$}
    \STATE Pick a set of random indices $\mathcal{B} \sim (\mathcal{U}[n])^{n_\text{mb}}$
    \STATE $d^t = \grad f_\mathcal{B}(\theta^{t}) - \grad f_\mathcal{B}(\tilde \theta) + \frac1n \sum_{j=1}^n \grad f_{j} (\tilde \theta)$
    \STATE $\omega^{t+1} \gets \theta^{t} - \gamma d^t$
    \STATE $\theta^{t+1} \gets \prox_{\gamma h}(\omega^{t+1})$
    \ENDFOR
    \STATE Update $\tilde\theta \gets \frac{1}{m_\text{mb}} \sum_{t=1}^{m_\text{mb}} \theta^t$, $\theta^0 \gets \tilde \theta$, $\tilde \theta^k \gets \tilde \theta$
    \ENDFOR%
    \STATE \textbf{Return:} $\tilde \theta^K$%
  \end{algorithmic}
\end{algorithm}


\paragraph{Baselines.}

We describe in this paragraph the algorithm that we put in competition in our experiments.

\begin{description}
  \item[FISTA] This is accelerated proximal gradient from~\cite{FISTA} with backtracking linesearch. Inner products necessary inside the backtracking are counted as well.

  \item[L-BFGS-B] A state-of-the-art quasi-Newton solver which provides a usually strong baseline for many batch optimization algorithms, see~\cite{LBFGS}.
  We use the original implementation of the algorithm proposed in \texttt{python}'s \texttt{scipy.optimize} module. Non-differentiability of the $\ell_1$-norm in the elastic-net penalization is dealt with the standard trick of reformulating the problem, using the fact that $|a| = a_+ + a_-$ for $a \in \R$.

  \item[HSVRG-UNIF-IMH] This is Algorithm~\ref{alg:hybrid} where \mbox{\textsc{ApproxMCMC}} is done via Algorithm~\ref{IMH} with uniform~proposal~$Q$.

  \item[HSVRG-ADAP-IMH] This is Algorithm~\ref{alg:hybrid} where \mbox{\textsc{ApproxMCMC}} is done via Algorithm~\ref{IMH} with adaptative~proposal~$Q = \pi^{\cdot}_{\tilde{\theta}}$.

  \item[HSVRG-AIS] This is Algorithm~\ref{alg:hybrid} where \mbox{\textsc{ApproxMCMC}} is done via Algorithm~\ref{AIS}, that is Adaptative~Importance~Sampling.




  \item[SVRG-MB] Mini-Batch Prox-SVRG described in~\cite{SVRG_MB}, which can be seen as Algorithm \ref{alg:hybrid} (see below) with $K_S = 0$.
  This is a \emph{simply} stochastic algorithm, since there is no MCMC approximation of the gradients $\grad f_i$. The question of mini-batch sizing is critical and is adressed in Appendix~\ref{appen:mb_size}. We used $n_\text{mb} = 0.1 n$ or $n_\text{mb} = 0.01 n$ in our experiments.
\end{description}

The ``simply stochastic'' counterpart SVRG-MB is way slower than the corresponding doubly stochatic versions, since they rely on many computations of stochastic gradients $\grad f_i$, which are numerically costly, as explained above.
The same settings are used throughout all experiments, some of them being tuned by hand: steps size for the variants of HSVRG are taken as $\gamma_t = \gamma_0 \in \{10^{-2},10^{-3},10^{-4}\}$ where $\gamma_0$ depends on the dataset, the phase length $m$ is equal to the number $n$ of failures of each datasets as suggested in~\cite{semi-sgd}.
As mentionned above, the doubly stochastic algorithms use different verions of \mbox{\textsc{ApproxMCMC}}.

\paragraph{Datasets}



We compare algorithms on the following datasets.
The first three are standard benchmarks in survival analysis, the fourth one is a large simulated dataset where the number of observations $n$ exceeds the number of features $d$.
This differs from supervised gene expression data: such a large-scale setting happens for longitudinal clinical trials, medical adverse event monitoring and business data minings tasks.

\begin{itemize}
\item NKI70 contains survival data for 144 breast cancer patients, 5 clinical covariates and the expressions from 70 gene signatures, see~\cite{van2002gene}.

\item Luminal contains survival data for 277 patients with breast cancer who received the adjuvant tamoxifen, with 44,928 expressions measurements, see~\cite{loi2007definition}.

\item Lymphoma contains 7399 gene expressions data for 240 lymphoma patients. The data was originally published in~\cite{alizadeh2000distinct}.

\item We generated a Gaussian features matrix $X$ with $n=10,000$ observations and $d=500$ predictors, with a Toeplitz covariance and correlation equal to $0.5$. The failure times follow a Weibull distribution. See Appendix~\ref{appen:simu} for details on simulation in this model.

\end{itemize}

We compare in Figures~\ref{fig:high_ridge1} and~\ref{fig:high_ridge2} all algorithms for ridge penalization, namely $\alpha = 0$ and $\lambda = 1 / \sqrt n$.
Experiences with other values of $\alpha$ and $\lambda$ are given in Appendix (including the Lasso penalization for instance).

\begin{figure}
  \centering
  \includegraphics[width=0.48\textwidth]{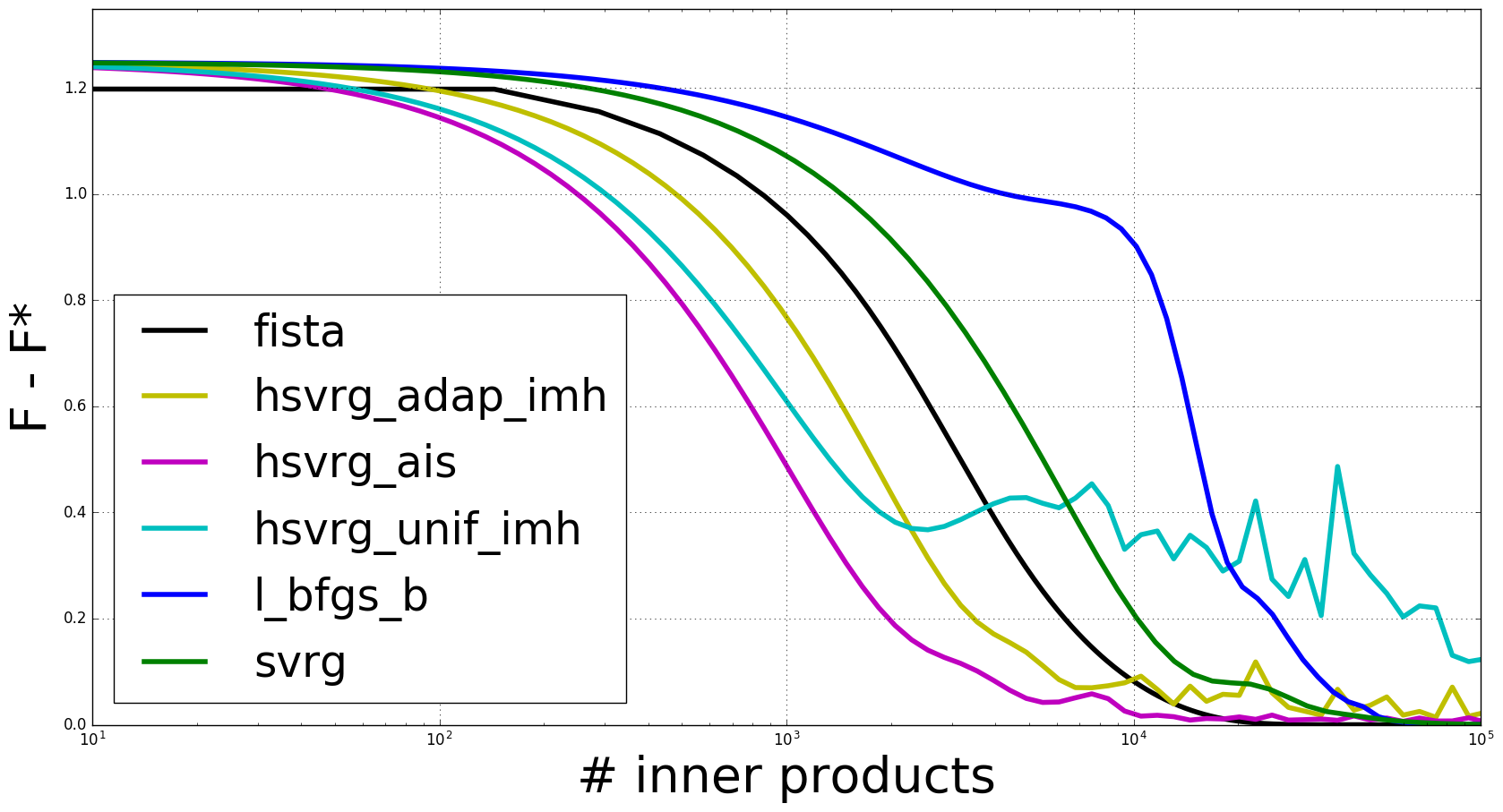}
  \includegraphics[width=0.48\textwidth]{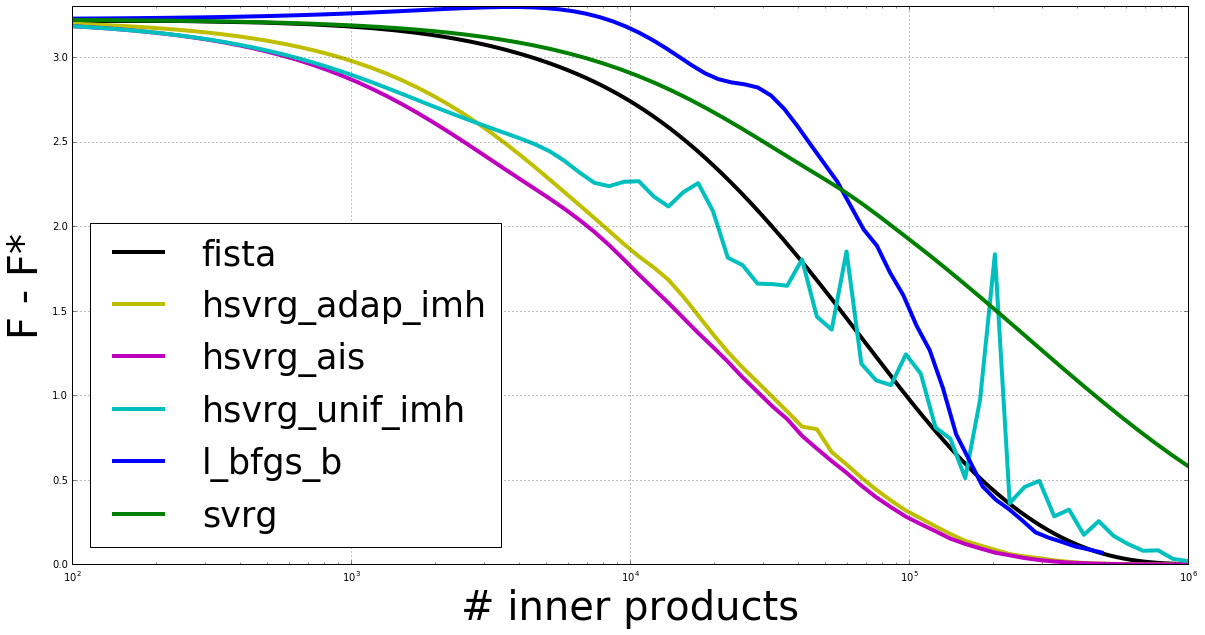}
  \caption{Distance to optimum of all algorithms on NKI70 (left) and Lymphoma (right) with ridge penalization ($\alpha = 0$ and $\lambda = 1 / \sqrt n$)}
  \label{fig:high_ridge1}
\end{figure}

\begin{figure}[htbp]
  \centering
  \includegraphics[width=0.48\textwidth]{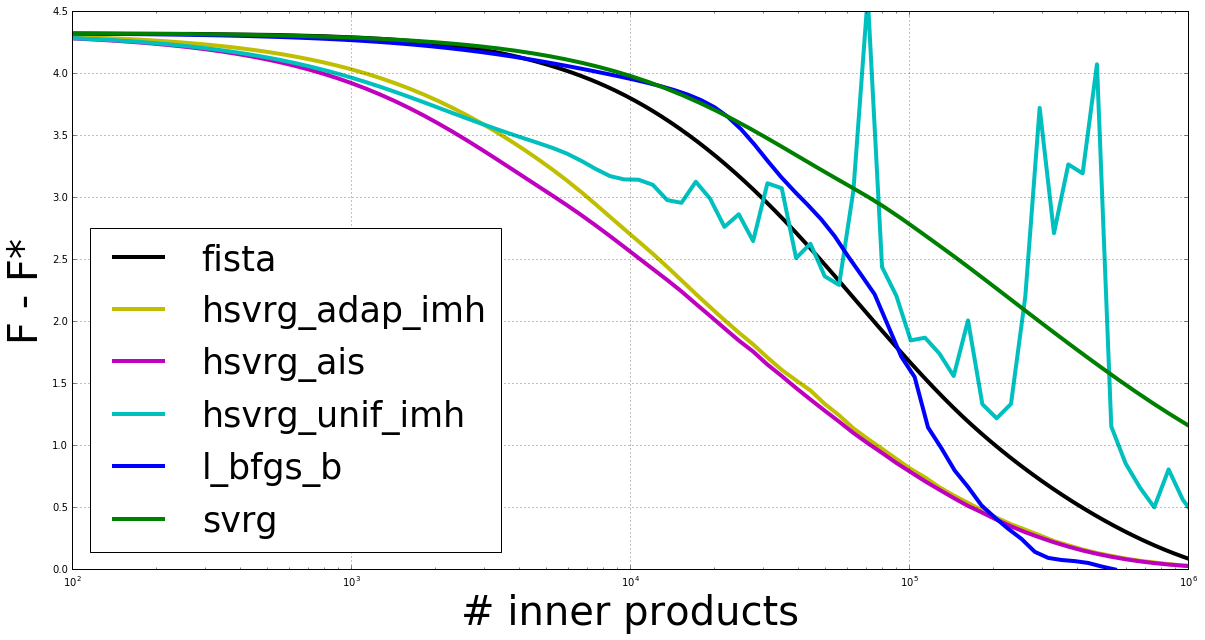}
  \includegraphics[width=0.48\textwidth]{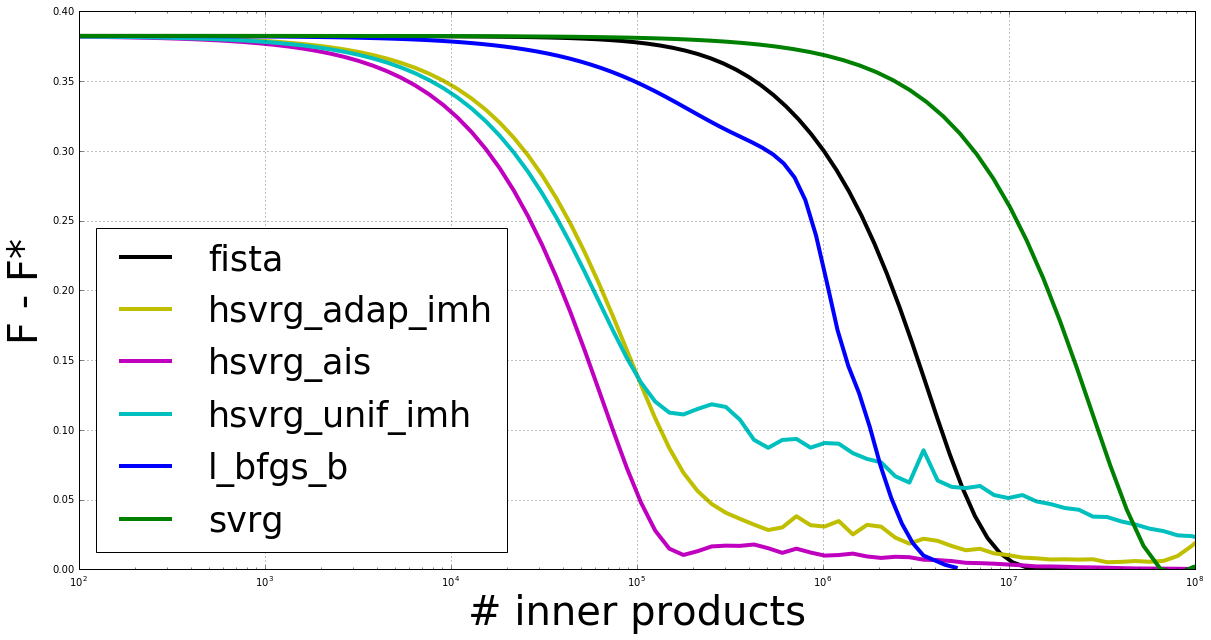}
  \caption{Distance to optimum of all algorithms on Luminal (left) and on the simulated dataset (right) with ridge penalization ($\alpha = 0$ and $\lambda = 1 / \sqrt n$)}
  \label{fig:high_ridge2}
\end{figure}

\paragraph{Conclusions.}


The experiments first show that the solvers HSVRG-ADAP-IMH and HSVRG-AIS give better results than HSVRG-UNIF-IMF. However, the HSVRG solvers behave particularly well during the first phases where the gradients can be noisy - due to a small number of iterations of the MCMC - and still point a decent descent direction.

\section{Conclusion}
\label{sec:conclusion}

 We have proposed a \emph{doubly} stochastic gradient algorithm to extend SGD-like algorithms beyond the empirical risk minimization setting. The algorithm we proposed is the result of two different ideas: sampling from uniform distribution to avoid the computation of a large sum, and sampling using MCMC methods to avoid the computation of a more complicated expectation. We have also provided theoretical guarantees of convergence for both the convex and the strongly-convex setting.

This \emph{doubly} stochastic gradient algorithm is very efficient during the early phases. The hybrid version of our algorithm, at the crossing of \emph{simply} and \emph{doubly} stochastic gradient algorithms, significantly outperforms state-of-the-art methods.

In a future work, we intend to extend our algorithm to Conditional Random Fields (CRF), where each subfunction's gradient takes the form
 \[
 \nabla f_i (\theta) = \nabla (- \log( p(y_i | x_i, \theta)) = \sum_{Y \in \mathcal{Y}_i} \frac{e^{H(X_i,Y)^\top \theta}}{\sum_{Y' \in \mathcal{Y}_i} e^{H(X_i,Y')^\top \theta}} (H(X_i,Y) - H(X_i,Y_i)),
 \]
 for a certain function $H$ (see Page 2 in~\cite{crf_schmidt2}).
Notice that the Cox negative partial likelihood can be seen as a particular case of CRF by setting $X_i = [x_j]_{j \in R_i} \in \mathbb{R}^{d \times |R_i|}$, $Y_i = [\mathbbm{1}_{j \in R_i}]_{j \in R_i} \in \{0,1\}^{|R_i|}$, $H(X,Y) = X Y$ and $\mathcal{Y}_i = \{ [\mathbbm{1}_{j=k}]_{j \in R_i}: k \in R_i \}$.

\newpage

\appendix

\section{Proofs}
\label{sec:proofs}

\subsection{Proof of Proposition~\ref{prop:IMH}}

We first prove Proposition~\ref{prop:IMH} that ensures that Algorithm~\ref{IMH} provides the bounds of~\ref{assump:2}.
\begin{proof}
Since there exists $M > 0$ such that the proposal $Q$ and the stationary distribution $\pi$ satisfy $\pi (x) \le M Q(x)$, for all $x$ in the support of $\pi$, the Theorem 7.8 in~\cite{robert2004monte} states that the Algorithm~\ref{IMH} produces a geometrically ergodic Markov kernel $P$ with ergodicity constants uniformly controlled:
\begin{align}
\|P^{k}(x, \cdot) - \pi \|_{TV} \le 2 \left( 1 - \frac{1}{M} \right)^k,
\end{align}
where $P^{k}$ is the kernel of the $k^{th}$ iteration of the algorithm and $\| \cdot \|_{TV}$ is the total variation norm. Since $\widehat \grad f_{i_t} (\theta^{t-1})$ is computed as the mean of the iterates of the Markov chain, a simple computation enables us to bound the bias of the error and Proposition 12 from~\cite{fort2003} gives the upper bound for the expected squared error:
\begin{align}
\|\mathbb{E}_t \eta^t\| \le \frac{C_1}{N_k} \mbox{ and } \mathbb{E}_t \| \eta^{t} \|^2  \le \frac{C_2}{N_k}
\end{align}
where $C_1$ and $C_2$ are some finite constants, and $N_k$ the number of iterations of the Markov chain. It can be shown that $C_1 = 2 M$ and that $C_2$ is related to a constant from the Rosenthal's inequality.
\end{proof}

\subsection{Preliminaries to the proofs of Theorems~\ref{thm:1} and~\ref{thm:2}}

In what follows, the key lemmas for the proofs of Theorems~\ref{thm:1} and~\ref{thm:2} are stated and proved when not directly borrowed from previous articles.

\begin{lemma}
\label{bound_var}
For $\Delta^t := \widehat \grad f_{i_t} (\theta^{t-1}) - \grad f_{i_t} (\tilde{\theta}) + \grad f(\tilde{\theta}) - \nabla f(\theta^{t-1})$, we have:
\begin{align*}
\mathbb{E}_t \| \Delta^t \|^2 &\le 8 L [ F(\theta^{t-1}) - F(\theta^*) \\
&+ F(\tilde{\theta}) - F(\theta^*) ] + 3 \mathbb{E}_t \| \eta^{t} \|^2.
\end{align*}
\end{lemma}

The proof of Lemma~\ref{bound_var} uses Lemma 1 in~\cite{SVRG}.
\begin{lemma}{\cite{SVRGold, SVRG}}
\label{lemma1}
Consider $F$ satisfying~\ref{assump:1}. Then,
\[
\frac{1}{n} \sum_{i=1}^n \| \nabla f_i (\theta) - \nabla f_i (\theta^*) \|^2 \le 2 L [ F(\theta) - F(\theta^*)]
\]
\end{lemma}

\begin{proof}[Proof of Lemma~\ref{bound_var}]
For the sake of simplicity, we now denote $d^t_i = \nabla f_{i} (\theta^{t-1}) - \nabla f_{i} (\tilde{\theta})$ and $d^t = \nabla f (\theta^{t-1}) - \nabla f (\tilde{\theta})$, so that one gets $\Delta^t = d^t_{i_t} - d^t + \eta^{t}$.
Then,
using the expectation introduced in Section~\ref{sec:theory},
we repeatedly use the identity $\mathbb{E}_t \| \xi \|^2 = \mathbb{E}_t \| \xi - \mathbb{E}_t \xi\|^2 + \|\mathbb{E}_t \xi\|^2$. First with
$\xi = \Delta^t$ (since $\mathbb{E}_t d^t_{i_t} = d^t$, one gets
 $\mathbb{E}_t \xi = \mathbb{E}_t \eta^{t}$)~:
\begin{align*}
 \mathbb{E}_t \| \Delta^t \|^2
 & = \mathbb{E}_t \| d^t_{i_t} + \eta^{t} - \left( d^t + \mathbb{E}_t \eta^{t} \right)\|^2 + \|\mathbb{E}_t \eta^{t}\|^2
\end{align*}
then, successively with $\xi = d^t_{i_t}  + \eta^{t}$, $\xi = d^t  + \eta^{t}$ and finally $\xi = \eta^{t}$ :
\begin{align*}
\mathbb{E}_t \| \Delta^t \|^2  &= \mathbb{E}_t \| d^t_{i_t} + \eta^{t}\|^2 + \|\mathbb{E}_t \eta^{t}\|^2 - \|d^t + \mathbb{E}_t \eta^{t}\|^2 \\
&= \mathbb{E}_t \| d^t_{i_t} + \eta^{t}\|^2 + \|\mathbb{E}_t \eta^{t}\|^2 \\
&- \left( \mathbb{E}_t \|d^t + \eta^{t}\|^2 - \mathbb{E}_t \|\eta^t - \mathbb{E}_t \eta^t \|^2 \right) \\
 & = \mathbb{E}_t \| d^t_{i_t} + \eta^{t}\|^2 + \mathbb{E}_t \|\eta^t\|^2- \mathbb{E}_t \|d^t + \eta^{t}\|^2 .
\end{align*}
Now we remark that $\mathbb{E}_t \|d^t + \eta^{t}\|^2 \ge 0$, and the identity $\|a+b\|^2 \le 2 \|a\|^2 + 2 \|b\|^2$ gives the majoration
\begin{align*}
 \mathbb{E}_t \| \Delta^t \|^2& \le  2 \mathbb{E}_t \| d^t_{i_t}\|^2 + 3 \mathbb{E}_t \|\eta^t\|^2.
\end{align*}
Now rewriting $d^t_{i_t} = \nabla f_{i_t} (\theta^{t-1}) - \nabla f_{i_t} (\theta^*) + \nabla f_{i_t} (\theta^*) - \nabla f_{i_t} (\tilde{\theta})$, the same identity leads to
\begin{align*}
  \mathbb{E}_t \| \Delta^t \|^2&\le  4 \mathbb{E}_t \| \nabla f_{i_t} (\theta^{t-1}) - \nabla f_{i_t} (\theta^*)\|^2 \\
  &+ 4 \mathbb{E}_t \| \nabla f_{i_t} (\tilde{\theta}) - \nabla f_{i_t} (\theta^*)\|^2 + 3 \mathbb{E}_t \|\eta^t\|^2.
\end{align*}
The desired result follows applying twice Lemma~\ref{lemma1}.
\end{proof}

When $F$ is $\mu$-strongly convex, the next Lemma (Lemma 3 in~\cite{SVRG}) provides a key lower bound.

\begin{lemma}{\cite{SVRG}}
\label{grad_map}
Consider $F = f + h$ satifying~\ref{assump:1}, where $f$ is $L_f$-smooth, $L_f>0$, $f$ is $\mu_f$-strongly convex, $\mu_f \ge 0$, $h$ is $\mu_h$-strongly convex, $\mu_h \ge 0$. For any $x, v \in \mathbb{R}^d$, we define $x^{+} = \prox_{\gamma h} (x - \gamma v)$, $g = \frac{1}{\gamma} (x - x^{+})$, where $\gamma \in (0,\frac{1}{L_f}]$. Then, for any $y \in \mathbb{R}^d$\textup:
\begin{align*}
F(y) & \ge F(x^{+}) + g^\top (y-x) + \frac{\gamma}{2} \|g\|^2 + \frac{\mu_f}{2} \|y-x\|^2  \\
&+ \frac{\mu_h}{2} \|y-x^{+}\|^2 + (v - \nabla f (x))^\top (x^{+} - y). \numberthis \label{lemme3}
\end{align*}
\end{lemma}

\begin{remark}
Note that in Lemma~\ref{grad_map}, one can freely choose $\mu_f$ and $\mu_h$ (in particular one can take $\mu_f = 0$ or $\mu_h = 0$), as long as $\mu_f + \mu_h = \mu$.
\end{remark}

The following Lemma comes from~\cite{moulines} (Lemma 14):
\begin{lemma}{\cite{moulines}}
\label{prox_x_star}
Consider $F = f + h$ satifying~\ref{assump:1}, where $f$ is $L_f$-smooth, and $T_\gamma : x \mapsto \prox_{\gamma h} [x - \gamma \nabla f (x)]$ with $\gamma \in (0, 2 / L_f]$. Let $x, y \in \mathbb{R}^d$, we have:
\begin{equation*}
\| T_\gamma (x) - T_\gamma (y) \| \le \|x - y\|
\end{equation*}
\end{lemma}

\subsection{Proof of Theorem~\ref{thm:1}}

\begin{proof}
The proof begins with the study of the distance $\|\theta^t - \theta^*\|^2$ \emph{between} the phases $k-1$ and $k$. To ease the reading, when staying between these two phases, we write $\tilde{\theta}$ instead of $\tilde{\theta}^{k-1}$. Introducing $g^t = \frac{1}{\gamma} (\theta^{t-1} - \theta^t)$, we may write:
\begin{align*}
\|\theta^t - \theta^*\|^2 & = \|\theta^{t-1} - \gamma g^t - \theta^*\|^2 \\
& = \|\theta^{t-1} - \theta^*\|^2  - 2 \gamma (g^t)^\top (\theta^{t-1} - \theta^*) \\
&+ \gamma^2 \|g^t\|^2.
\end{align*}
To upper bound the term $- 2 \gamma (g^t)^\top (\theta^{t-1} - \theta^*) + \gamma^2 \|g^t\|^2$, we apply the Lemma~\ref{grad_map} with $x = \theta^{t-1}$, $x^{+} = \theta^t$ and $y = \theta^*$. With again $\Delta^t = \widehat \grad f_{i_t} (\theta^{t-1}) - \grad f_{i_t} (\tilde{\theta}) + \grad f(\tilde{\theta}) - \nabla f(\theta^{t-1})$, we obtain
\begin{align*}
- (g^t)^\top & (\theta^{t-1} - \theta^*)  + \frac{\gamma}{2} \|g^t\|^2 \\
&\le F(\theta^*) - F(\theta^t) - \frac{\mu_f}{2} \| \theta^{t-1} - \theta^*\|^2 \\
&- \frac{\mu_h}{2} \| x^{t} - \theta^*\|^2 - (\Delta^t)^\top (\theta^t - \theta^*),
\end{align*}
and
\begin{align*}
\|\theta^t - \theta^*\|^2 & \le \|\theta^{t-1} - \theta^*\|^2 + 2 \gamma [F(\theta^*) - F(\theta^t)] \\
& - 2 \gamma (\Delta^t)^\top (\theta^t - \theta^*) \numberthis \label{diff_norm}.
\end{align*}
We now concentrate on the quantity $- 2 \gamma (\Delta^t)^\top (\theta^t - \theta^*)$. Introducing $\nu^t = \mbox{prox}_{\gamma h} [\theta^{t-t1} - \gamma \nabla f(\theta^{t-1})] \in \mathcal{F}_{t-1}$ i.e. the vector obtained from $\theta^{t-1}$ after an exact proximal gradient descent step, we get
\begin{align*}
- 2 \gamma (\Delta^t)^\top & (\theta^t - \theta^*) \\
& = - 2 \gamma (\Delta^t)^\top (\theta^t - \nu^t) - 2 \gamma (\Delta^t)^\top (\nu^t - \theta^*)\\
& \le 2 \gamma \|\Delta^t\| \cdot \|\theta^t - \nu^t\| - 2 \gamma (\Delta^t)^\top (\nu^t - \theta^*)
\end{align*}
where the inequality follows from the Cauchy-Schwartz inequality. Now the non-expansiveness property of proximal operators $\|\mbox{prox}_{\gamma h} (x) - \mbox{prox}_{\gamma h} (y) \| \le \|x-y\|$ leads to
\begin{align*}
- 2 \gamma & (\Delta^t)^\top (\theta^t - \theta^*) \\
& \le 2 \gamma \|\Delta^t\| \cdot \|\{ \theta^{t-1} - \gamma (\Delta^t + \nabla f(\theta^{t-1})) \} \\
&- \{ \theta^{t-1} - \gamma \nabla f(\theta^{t-1}) \}\| - 2 \gamma (\Delta^t)^\top (\nu^t - \theta^*) \\
& \le 2 \gamma^2 \|\Delta^t\|^2 - 2 \gamma (\Delta^t)^\top (\nu^t - \theta^*).
\end{align*}
Reminding that $\nu^t \in \mathcal{F}_{t-1}$, we derive:
\begin{align*}
- 2  \gamma & \mathbb{E}_t (\Delta^t)^\top (\theta^t - \theta^*) \\
&\le  2 \gamma^2 \mathbb{E}_t \|\Delta^t\|^2
 - 2 \gamma ( \mathbb{E}_t \Delta^t)^\top (\nu^t - \theta^*)\\
& \le  2 \gamma^2 \mathbb{E}_t \|\Delta^t\|^2
 + 2 \gamma \|\mathbb{E}_t \Delta^t\| \cdot \|\nu^t - \theta^*\|,
\end{align*}
the last inequality comes from the Cauchy-Schwartz inequality.
Since $\theta^*$ is the minimum of $F = f + h$, it satisfies $\theta^* = \mbox{prox}_{\gamma h} [\theta^* - \gamma \nabla f (\theta^*)]$. Thus, the Lemma~\ref{prox_x_star} and the~\ref{assump:1} on the sequence $(\theta^t)$ give us $\|\nu^t - \theta^*\| \le \|\theta^{t-1} - \theta^*\| \le B$. We also remark that $\mathbb{E}_t \Delta^t = \mathbb{E}_t \eta^t$. For all $t$ between phases $k-1$ and $k$, we finally apply Lemma~\ref{bound_var} to obtain:
\begin{align}
\label{youyou1}
- 2 \gamma & \mathbb{E}_t (\Delta^t)^\top (\theta^t - \theta^*) \nonumber\\
&\le 16 \gamma^2 L [ F(\theta^{t-1}) - F(\theta^*) + F(\tilde{\theta}) - F(\theta^*) ] \nonumber\\
& + 6 \gamma^2 \mathbb{E}_t \| \eta^{t} \|^2 + 2 \gamma B \|\mathbb{E}_t \eta^t\|.
\end{align}
Taking the expectation $\mathbb{E}_t$ on inequation \eqref{diff_norm} and combining with previous inequality leads to
\begin{align*}
\mathbb{E}_t \|\theta^t - & \theta^*\|^2  \le \|\theta^{t-1} - \theta^*\|^2 + 2 \gamma [F(\theta^*) - F(\theta^t)] \\
& + 16 \gamma^2 L [ F(\theta^{t-1}) - F(\theta^*) + F(\tilde{\theta}) - F(\theta^*) ] \\
& + 6 \gamma^2 \mathbb{E}_t \| \eta^{t} \|^2 + 2 \gamma B \|\mathbb{E}_t \eta^t\|.
\end{align*}
With the notation of Algorithm~\ref{alg:s2vrg}, $ \tilde{\theta} =\tilde{\theta}^{k-1} = \theta^0$. Now, applying iteratively the previous inequality over $t = 1,2,\ldots,m$ and taking the expectation $\mathbb{E}$ over $i_1, \theta^1, i_2, \theta^2, \ldots, i_m, \theta^m$, we obtain:
\begin{align*}
\mathbb{E} \| & \theta^m  - \theta^*\|^2  + 2 \gamma [ \mathbb{E} F(\theta^m) - F(\theta^*) ] \\
+ 2 & \gamma (1 - 8 L \gamma) \sum_{t=1}^{m-1} [ \mathbb{E} F(\theta^t) - F(\theta^*)] \\
&\le \|\theta^0 - \theta^*\|^2 + 16 L \gamma^2 [ F(\theta^0) - F(\theta^*) + m ( F(\tilde{\theta}) \\
&- F(\theta^*) )] + 6 \gamma^2 \sum_{t=1}^{m} \mathbb{E} \| \eta^{t} \|^2 + 2 \gamma B \sum_{t=1}^{m} \mathbb{E} \|\mathbb{E}_t \eta^t\|.
\end{align*}
Now, by convexity of $F$ and the definition $\tilde{\theta}^{k} = \frac{1}{m} \sum_{t=1}^m \theta^t$,  we may write $F(\tilde{\theta}^{k}) \le \frac{1}{m} \sum_{t=1}^m F(\theta^t)$.
Noticing that $2 \gamma (1 - 8 L \gamma) < 2 \gamma$ leads to
\begin{align*}
2 \gamma & (1 - 8 L \gamma) m [ \mathbb{E} F(\tilde{\theta}^{k}) - F(\theta^*)] \\
&\le \| \tilde{\theta} - \theta^* \|^2 + 16 L \gamma^2 (m+1) [ F(\tilde{\theta}) - F(\theta^*)] \\
&+ 6 \gamma^2 \sum_{t=1}^{m} \mathbb{E} \| \eta^{t} \|^2 + 2 \gamma B \sum_{t=1}^{m} \|\mathbb{E} \eta^t\|.
\end{align*}
Under the~\ref{assump:2}, we have
\begin{align*}
  6 \gamma^2 \sum_{t=1}^{m} & \mathbb{E} \| \eta^{t} \|^2 + 2 \gamma B \sum_{t=1}^{m} \|\mathbb{E} \eta^t\| \\
  &\leq( 6 \gamma^2 C_2 + 2 \gamma B C_1 ) \frac{m}{N_k}
\end{align*}
whereas the $\mu$-strong convexity of $F$ implies $\| \tilde{\theta}^{k-1} - \theta^* \|^2 \le \frac{2}{\mu} [ F(\tilde{\theta}^{k-1}) - F(\theta^*)]$. This leads to
\begin{align*}
\mathbb{E} F(\tilde{\theta}^{k}) - F(\theta^*) \le \rho \left( \mathbb{E} F(\tilde{\theta}^{k-1}) - F(\theta^*) \right) + \frac{D}{N_k}
\end{align*}
for $D$ and $\rho$ as defined in the theorem. Applying the last inequality recursively leads to the result.
\end{proof}

\subsection{Proof of Theorem~\ref{thm:2}}
\begin{proof}
As at the begining of the proof of Theorem~\ref{thm:1}, we consider that we stand between phase $k-1$ and phase $k$ of Algorithm~\ref{alg:s2vrg} and consequently $\theta^0 = \tilde \theta^{k-1}$.
We use the same arguments until~\eqref{diff_norm}, with the difference that, in this non-strongly convex case, we have $\mu_f=\mu_h=0$. We obtain for all $t$ between phases $1$ and $m$
\begin{align*}
F(\theta^t) - F(\theta^*) \le \frac{1}{2 \gamma} & (\|\theta^{t-1} - \theta^*\|^2 - \|\theta^{t} - \theta^*\|^2) \\
&- (\theta^t - \theta^*)^\top  \Delta^t.
\end{align*}
Summing over $t = 1,\ldots,\tau$ (for $\tau \leq m$) leads to
\begin{align*}
\sum_{t=1}^\tau [F(\theta^t) & - F(\theta^*)] \le \frac{1}{2 \gamma} (\sum_{t=0}^{\tau-1} \|\theta^{t} - \theta^*\|^2 \\
&-\sum_{t=1}^{\tau} \|\theta^{t} - \theta^*\|^2) - \sum_{t=1}^\tau (\theta^t - \theta^*)^\top \Delta^t.  \numberthis \label{eq:sum}
\end{align*}
We now use Equation~\eqref{eq:sum} (with $\tau=m$) and the convexity of $\|\cdot\|^2$ with $\tilde{\theta}^k =\frac{1}{m} \sum_{t=1}^m \theta^t $ to write
\begin{align*}
\sum_{t=1}^m [F(\theta^t) & - F(\theta^*)] \\
\le \frac{1}{2 \gamma} & \left( \sum_{t=0}^{m-1} \|\theta^{t} - \theta^*\|^2 - m \|\tilde{\theta}^k - \theta^*\|^2 \right) \\
&- \sum_{t=1}^m (\theta^t - \theta^*)^\top \Delta^t. \numberthis \label{eq:thm2}
\end{align*}
Starting from Equation \eqref{eq:sum} again but now summing over $l=1,\ldots,t$, we get
\begin{align*}
\frac{1}{2 \gamma} ( \|\theta^{0} - \theta^*\|^2 - & \|\theta^{t} - \theta^*\|^2) - \sum_{l=1}^t (\theta^l - \theta^*)^\top \Delta^l \\
&\geq  \sum_{l=1}^t [F(\theta^l) - F(\theta^*)] \\
&\geq 0,\numberthis \label{eq:maj_thm2}
\end{align*}
where the last inequality follows from the definition of $\theta^*$.
In \eqref{eq:thm2}, we now substitute $\|\theta^{t} - \theta^*\|^2$ by the upper bound derived from~\eqref{eq:maj_thm2} to write (noticing that $\theta^0=~\tilde{\theta}^{k-1}$):
\begin{align*}
\sum_{t=1}^m [F(\theta^t) &- F(\theta^*) ] \\
\le \frac{m}{2 \gamma} (\|&\tilde{\theta}^{k-1} - \theta^*\|^2 - \|\tilde{\theta}^k - \theta^*\|^2) \\
&- \sum_{t=1}^{m-1} \sum_{l=1}^{t} (\theta^l - \theta^*)^\top \Delta^l - \sum_{t=1}^m (\theta^t - \theta^*)^\top \Delta^t \\
\le \frac{m}{2 \gamma} (\|&\tilde{\theta}^{k-1} - \theta^*\|^2 - \|\tilde{\theta}^k - \theta^*\|^2) \\
&- \sum_{t=1}^{m} (m+1-t) (\theta^t - \theta^*)^\top \Delta^t.
\end{align*}
As in the proof of Theorem~\ref{thm:1} (see  Equation~\eqref{youyou1}), each term $-\mathbb E_t (\theta^t - \theta^*)^\top \Delta^t$ is upper bounded by $8 \gamma L [ F(\theta^{t-1}) - F(\theta^*) + F(\tilde{\theta}^{k-1}) - F(\theta^*) ]  + 3 \gamma \mathbb{E}_t || \eta^{t} ||^2 +  B ||\mathbb{E}_t \eta^t||$.
Now with $m+1-t \leq m$ and~\ref{assump:2}, we obtain:
\begin{align*}
\frac{1}{m} &\sum_{t=1}^m \mathbb{E} [F(\theta^t) - F(\theta^*)] \\
& \le \frac{1}{2 \gamma} (\|\tilde{\theta}^{k-1} - \theta^*\|^2 - \mathbb{E} \|\tilde{\theta}^k - \theta^*\|^2) \\
& + 8 L \gamma \big\{ \sum_{t=1}^m [ \mathbb{E} F(\theta^{t-1}) - F(\theta^*)] + F(\theta^*)-\mathbb E[ F(\theta^m)] \\
&+(m+1)[ \mathbb E[F(\tilde{\theta}^{k-1})] -  F(\theta^*)] \big\}  + m \frac{3 \gamma C_2 + B C_1}{N_k} .
\end{align*}
By definition of $\gamma$, we have $8 L m \gamma < 1$, and we can use the convexity of $F$ to lower bound the left hand side. With the inequality $\mathbb{E} [F(\theta^m)] - F(\theta^*) \ge 0$, one has:
\begin{align*}
(1 &- 8 L \gamma m) \left[ \mathbb{E} [F(\tilde{\theta}^k)] - F(\theta^*) \right] \\
& \le \frac{1}{2 \gamma} \left( \|\tilde{\theta}^{k-1} - \theta^*\|^2 - \mathbb{E} \|\tilde{\theta}^k - \theta^*\|^2 \right) \\
& + 8 L \gamma (m+1) \left[ \mathbb{E} [F(\tilde{\theta}^{k-1})] - F(\theta^*) \right] \\
&+ m \frac{3 \gamma C_2 + B C_1}{N_k}
\end{align*}
We now take the expectation $\mathbb{E}$ on all iterates of the algorithm i.e. on the iterates $i_1, \theta^1, i_2, \theta^2, \ldots, i_m, \theta^m$ from the first phase. Introduce the notations $A^k = \mathbb{E} [F(\tilde{\theta}^k)] - F(\theta^*)$ and $a=(8 L \gamma ( m + 1))/(1 - 8 L m \gamma)< 1$ , last inequality leads to:
\begin{align*}
A^k & - a A^{k-1} \\
&\le \frac{1}{2 \gamma (1 - 8 L m \gamma)} \left( \mathbb{E} \|\tilde{\theta}^{k-1} - \theta^*\|^2 - \mathbb{E} \|\tilde{\theta}^k - \theta^*\|^2 \right) \\
&+ \frac{D}{N_k},
\end{align*}where $D$ is defined in the theorem.
Summing over the phases $k=1,2,\ldots,K+1$ and lower bounding $A^{K+1}$ with $0$, we obtain:
\begin{align*}
&(1 - a) \sum_{k=1}^{K} A^k & \\&\le a A^0 + \frac{1}{2 \gamma (1 - 8 L m \gamma)} \|\tilde{\theta}^{0} - \theta^*\|^2 + \sum_{k=1}^{K+1} \frac{D}{N_k}
\end{align*}
The last argument is the use of the convexity of $F$. Remark the explicit forms of the constants in the theorem:
\begin{equation*}D_1 = \frac{a}{1 - a} A^0 + \frac{1}{1 - a} \frac{\|\tilde{\theta}^{0} - \theta^*\|^2}{2 \gamma (1 - 8 L m \gamma)}
\end{equation*} and $D_2 = \frac{D}{1 - a}$.
\end{proof}

\section{Supplementary experiments}
\label{sec:section_name}

We have tested all algorithms with other settings for the penalization. Namely, we considered:
\begin{description}
  \item[High lasso.] We take $\alpha = 1$ and $\lambda = 1 / \sqrt{n}$ and illustrate our results in Figure~\ref{fig:high_lasso}.
  \item[Low lasso.] We take $\alpha = 1$ and $\lambda = 1 / n$ and illustrate our results in Figure~\ref{fig:low_lasso}.
  \item[High ridge.] We take $\alpha = 0$ and $\lambda = 1 / \sqrt{n}$ and illustrate our results in Figures~\ref{fig:high_ridge1} and Figures~\ref{fig:high_ridge2}.
  \item[Low ridge.] We take $\alpha = 0$ and $\lambda = 1 / n$ and illustrate our results in Figure~\ref{fig:low_ridge}.
\end{description}

\begin{figure}
  \centering
  \includegraphics[width=0.48\textwidth]{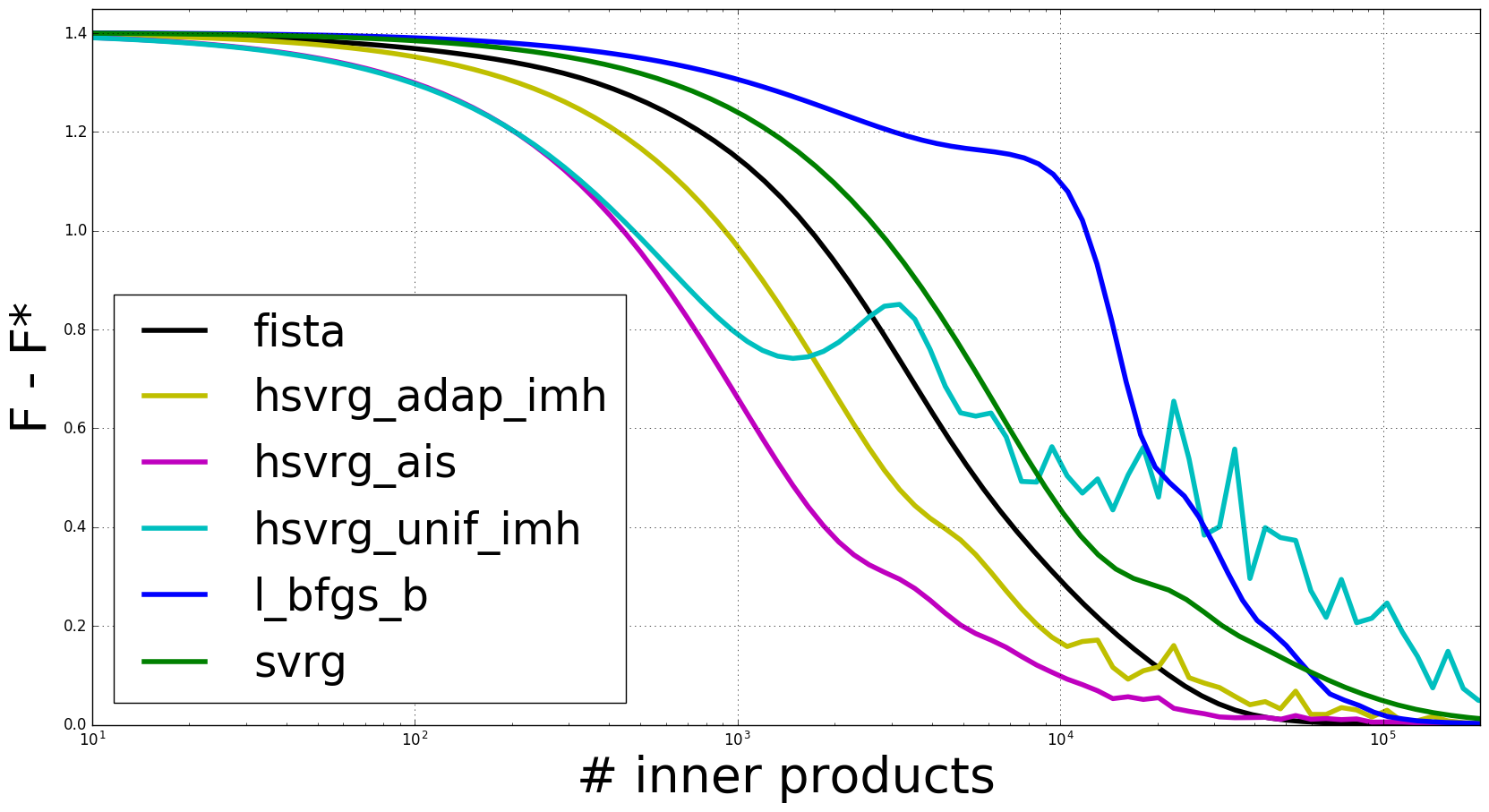}
  \includegraphics[width=0.48\textwidth]{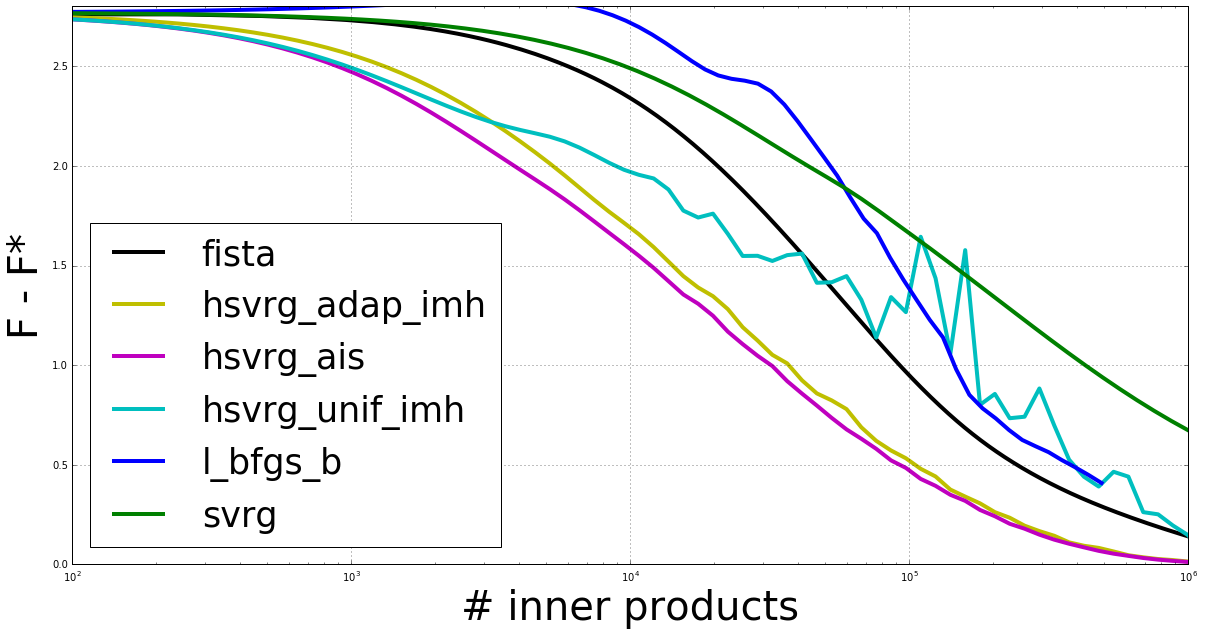}
  \includegraphics[width=0.48\textwidth]{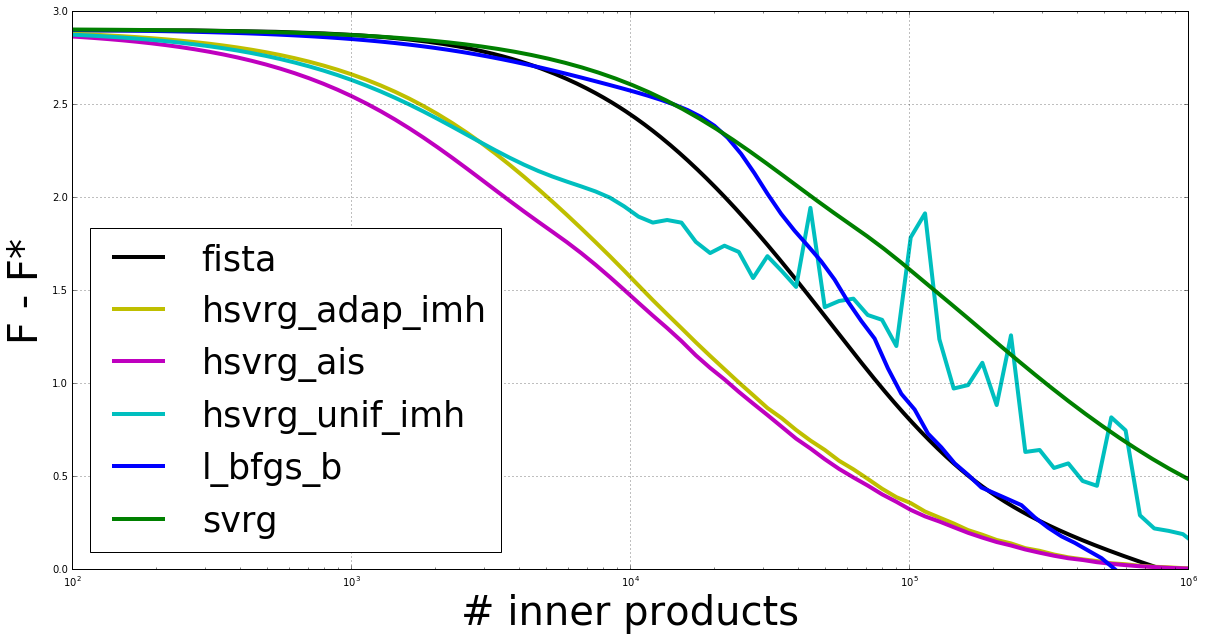}
  \includegraphics[width=0.48\textwidth]{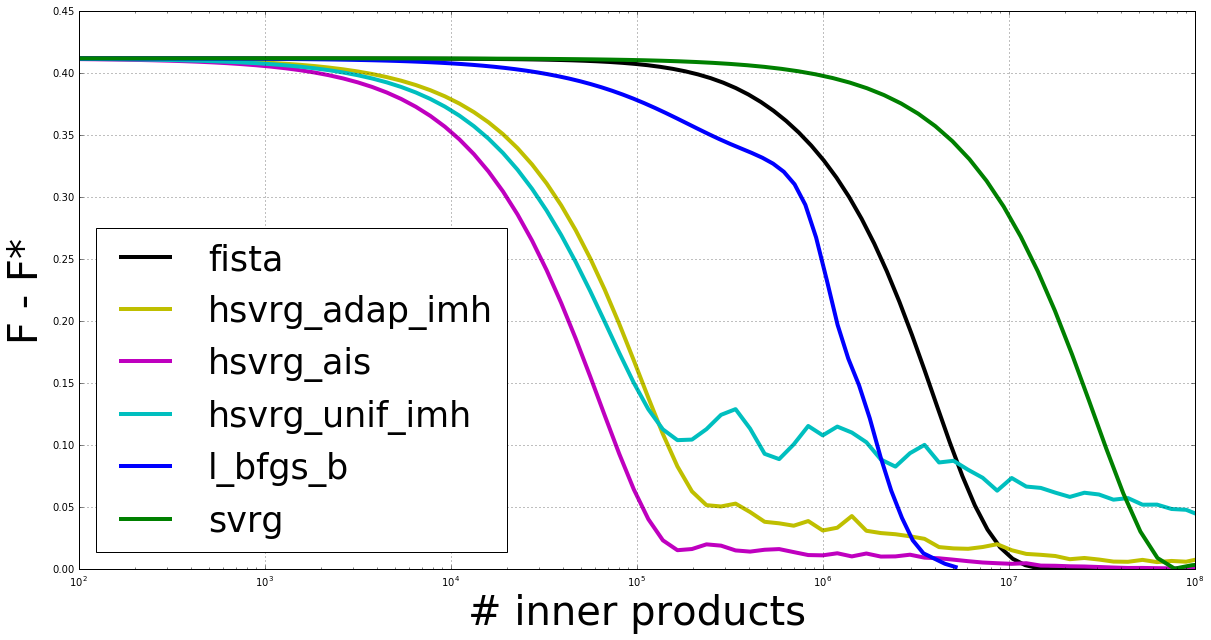}
  \caption{Distance to optimum of all algorithms on NKI70, Lymphoma, Luminal and on the simulated dataset (respectively from top to bottom) for \textbf{Low-ridge} penalization}
  \label{fig:low_lasso}
\end{figure}

\begin{figure}
  \centering
  \includegraphics[width=0.48\textwidth]{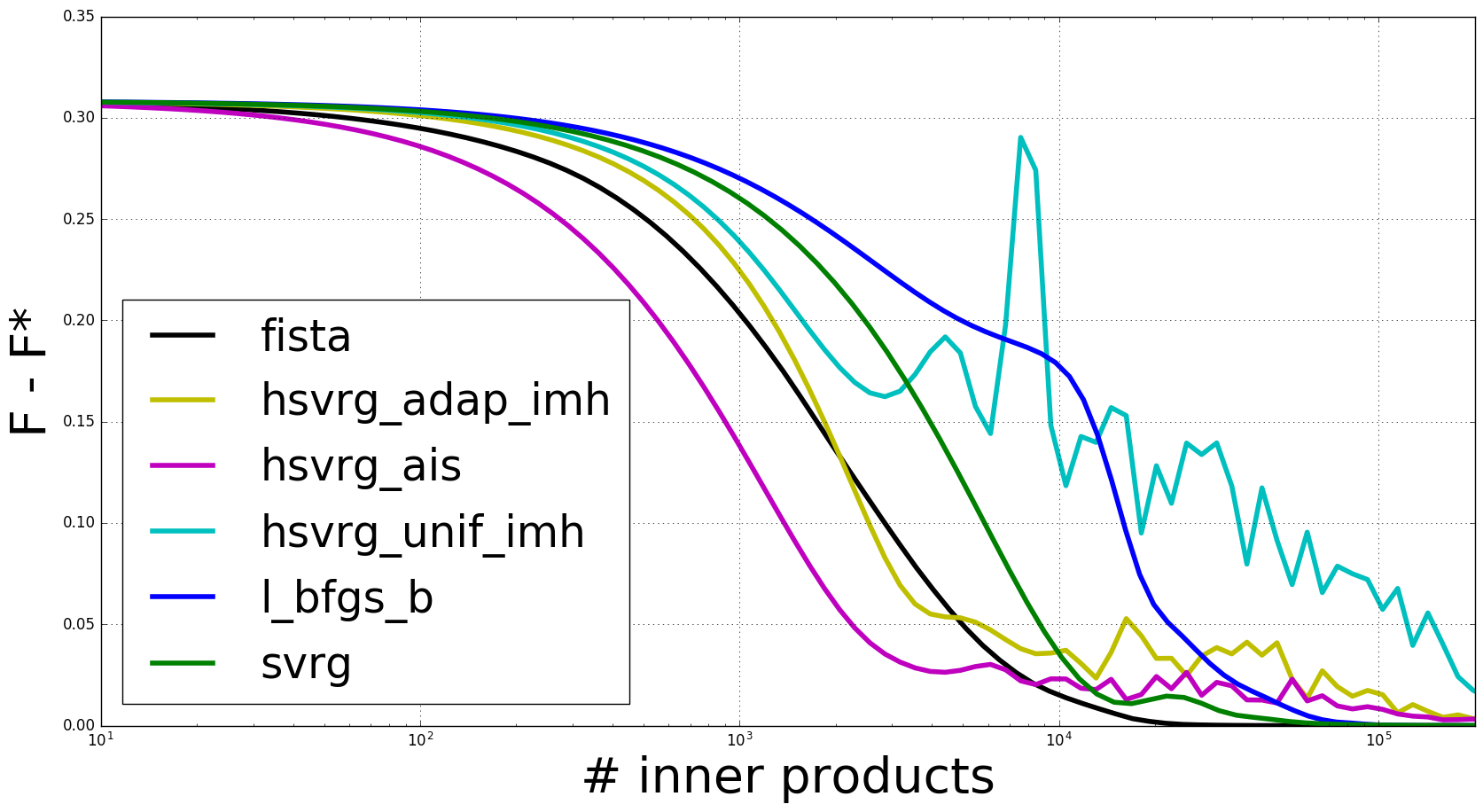}
  \includegraphics[width=0.48\textwidth]{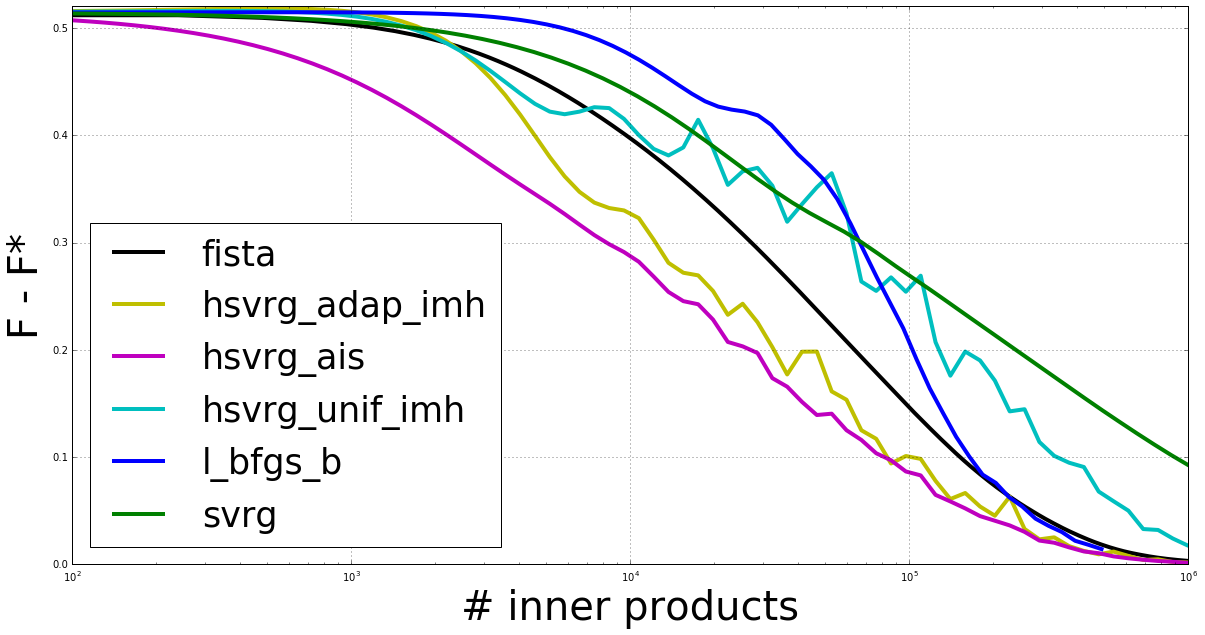}
  \includegraphics[width=0.48\textwidth]{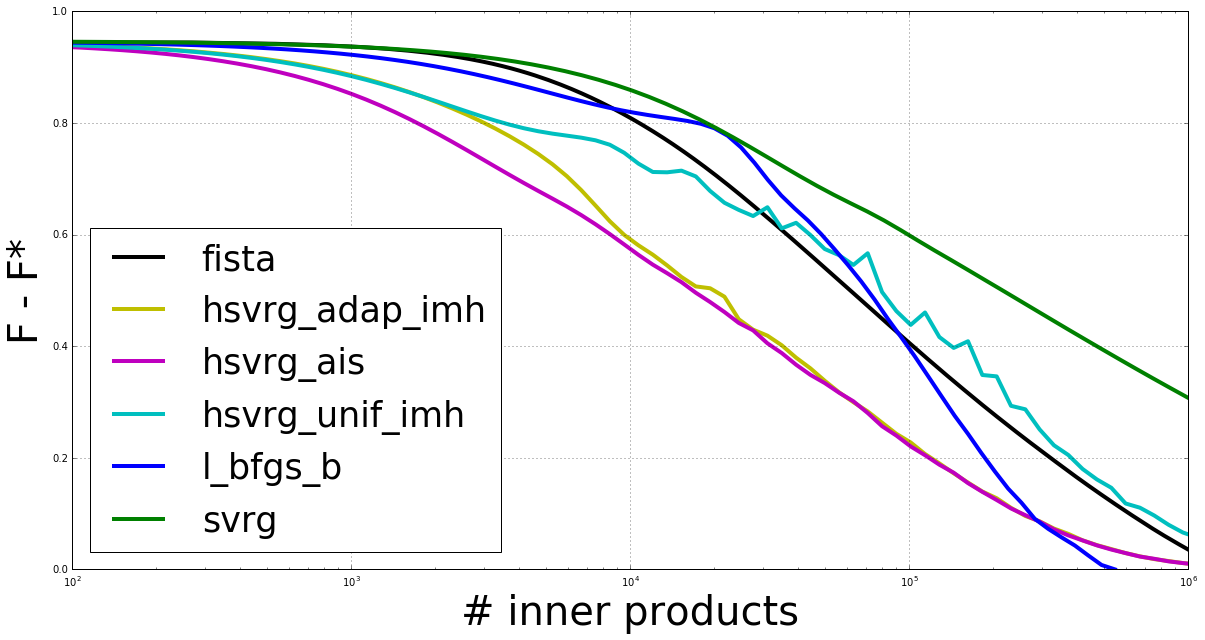}
  \includegraphics[width=0.48\textwidth]{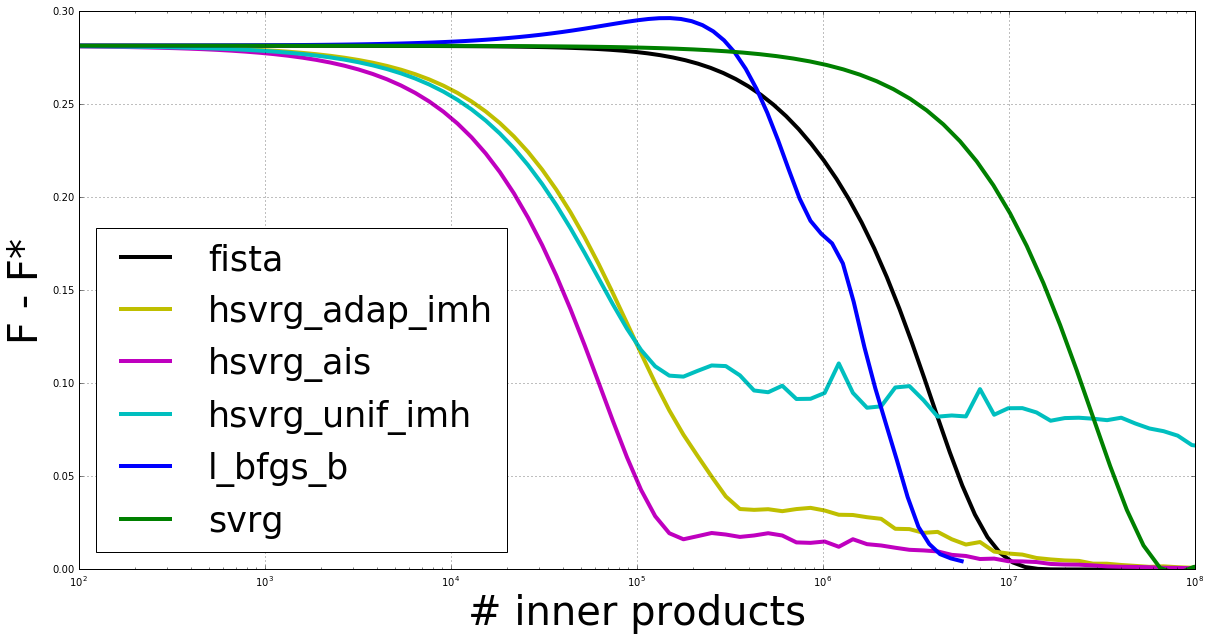}
  \caption{Distance to optimum of all algorithms on NKI70, Lymphoma, Luminal and on the simulated dataset (respectively from top to bottom) for \textbf{High-lasso} penalization}
  \label{fig:high_lasso}
\end{figure}

\begin{figure}
  \centering
  \includegraphics[width=0.48\textwidth]{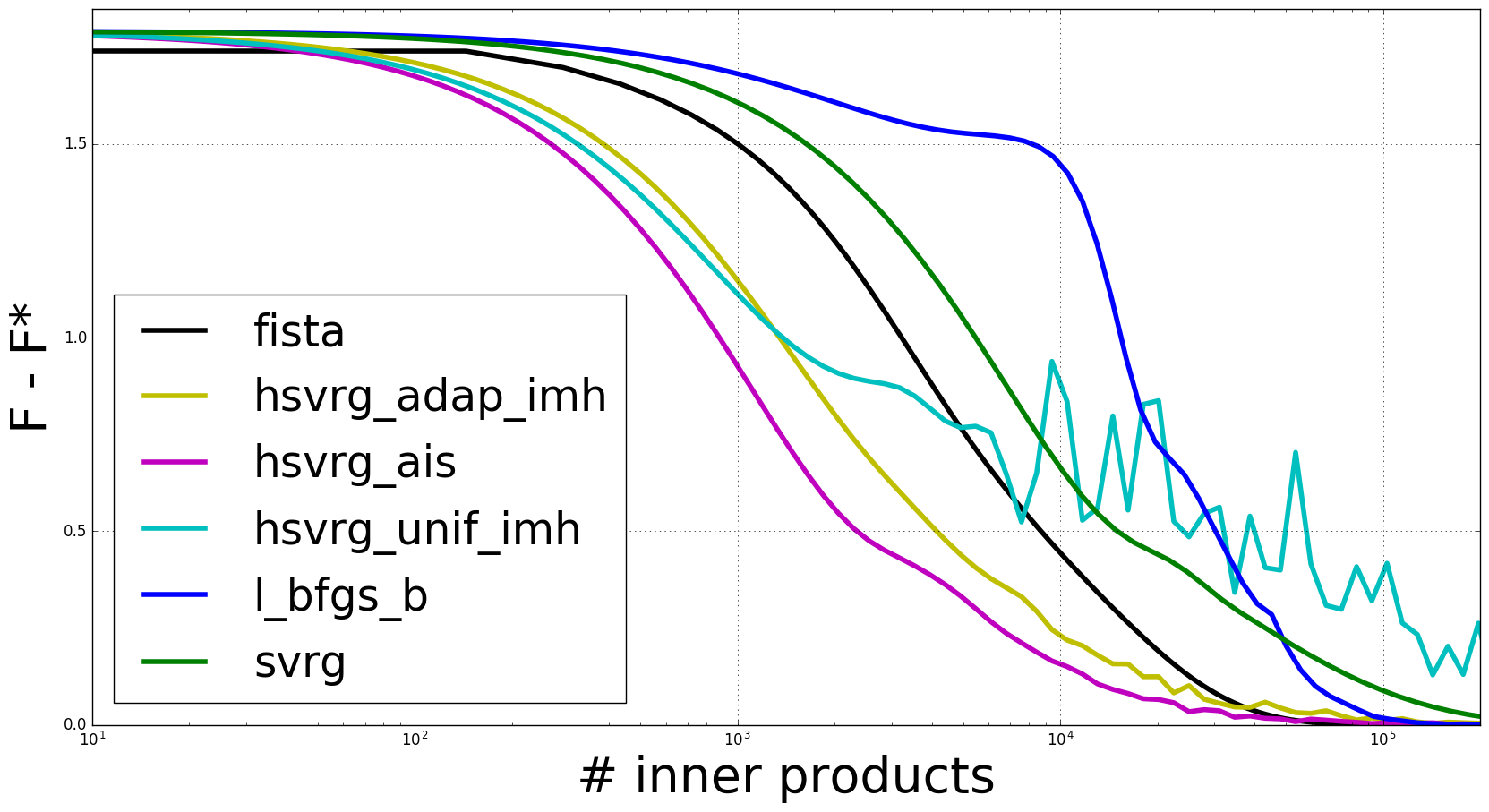}
  \includegraphics[width=0.48\textwidth]{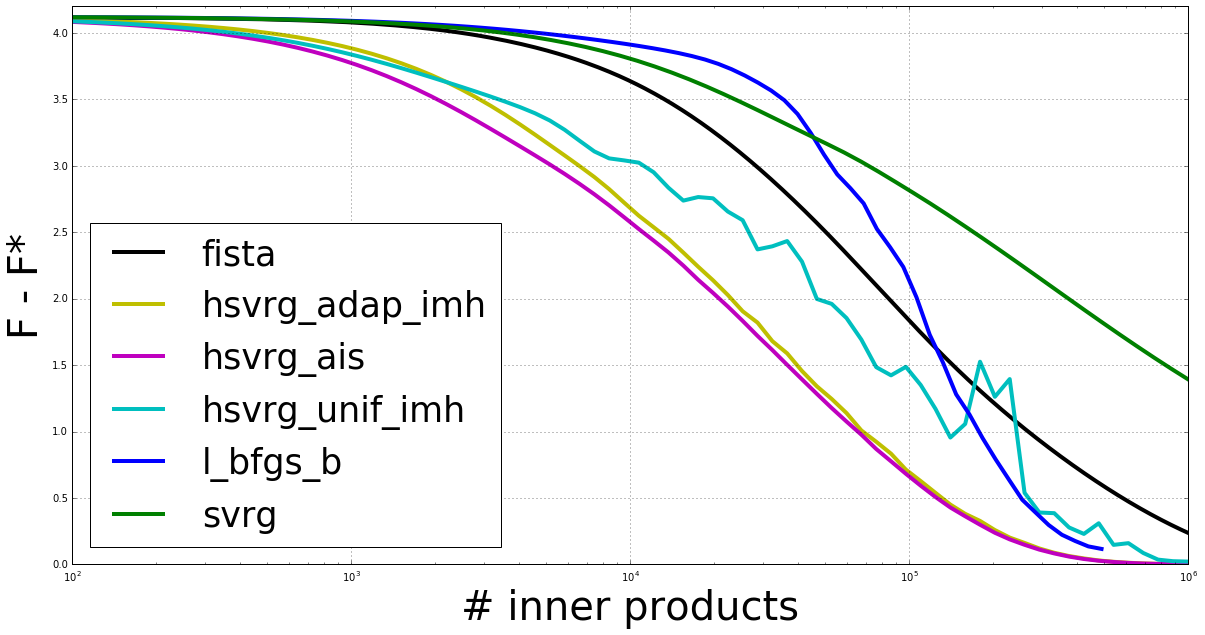}
  \includegraphics[width=0.48\textwidth]{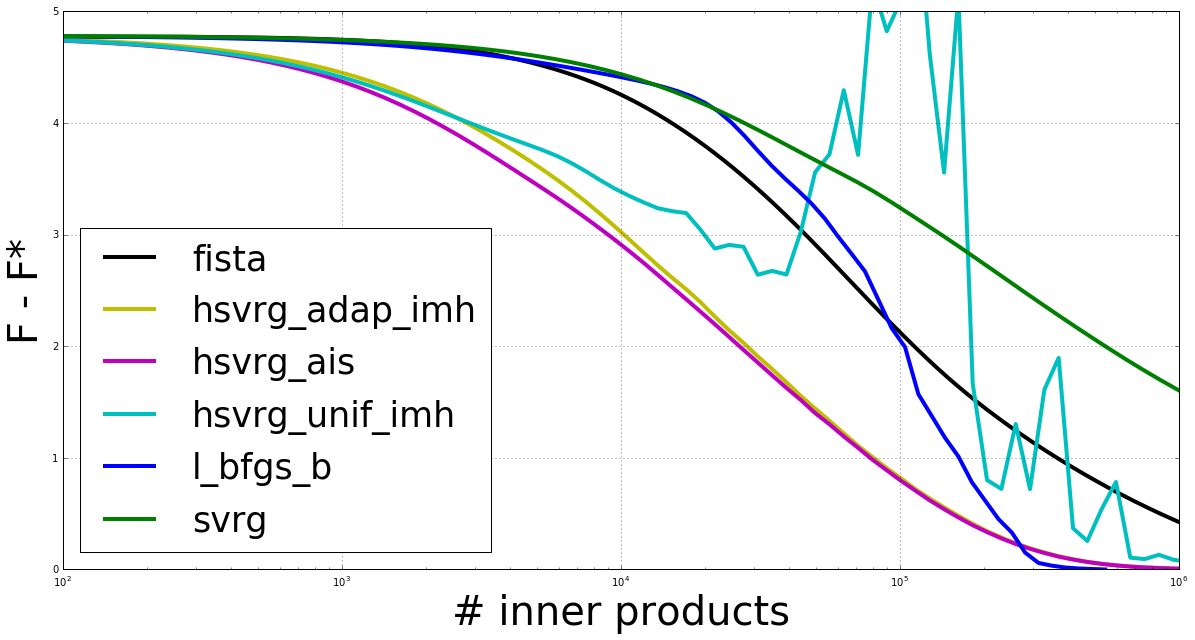}
  \includegraphics[width=0.48\textwidth]{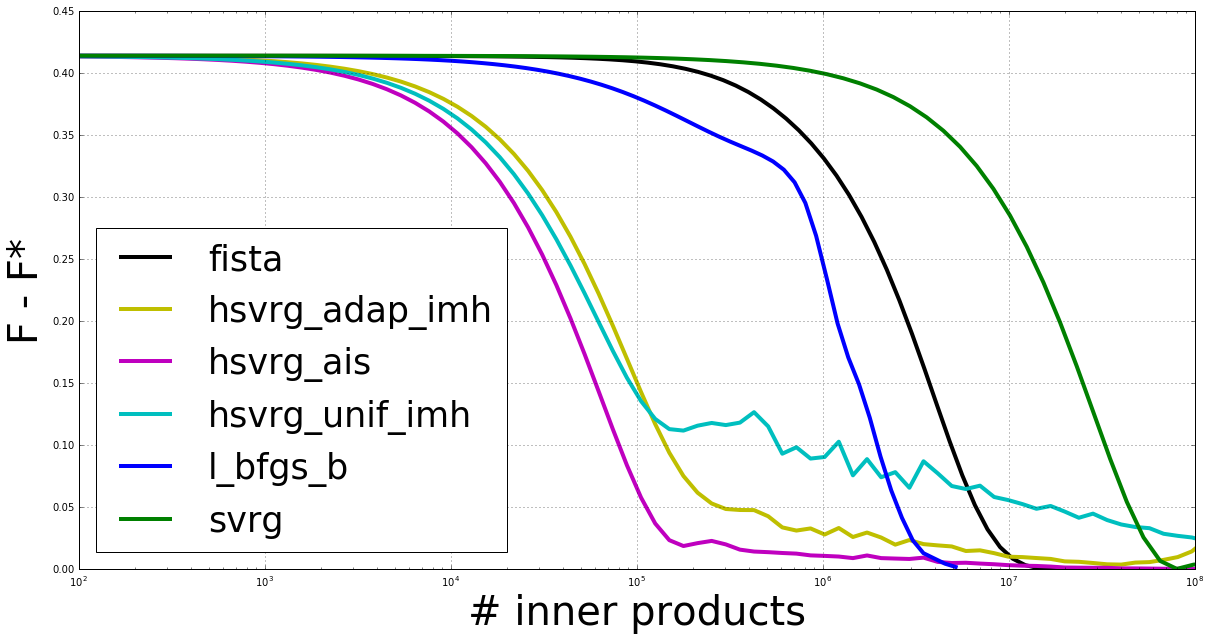}
\caption{Distance to optimum of all algorithms on NKI70, Lymphoma, Luminal and on the simulated dataset (respectively from top to bottom) for \textbf{Low-lasso} penalization}
  \label{fig:low_ridge}
\end{figure}

\section{Simulation of data}
\label{appen:simu}

With Cox model, the hazard ratio for the failure time $T_i$ of the $i^{th}$ patient takes the form:
\begin{equation*}
\lambda_i (t)=\lambda_0(t) \exp(x_i^\top \theta),
\end{equation*}
where $\lambda_0 (t)$ is a baseline hazard ratio, and $x_i \in \mathbb{R}^d$ the covariates of the $i^{th}$ patient. \\
We first simulate the feature matrix $X\in \mathbb{R}^{n \times d}$ as a Gaussian vector with a Toepliz covariance, where the correlation between features $j$ and $j'$ is equal to $\rho^{|j - j'|}$, for some $\rho \in (0, 1)$.

We want now to simulate the observed time $y_i$ that corresponds to $x_i$.
We denote the cumulative hazard function $\Lambda (t) = \int_0^t \lambda (s) ds$. Using the definition $\lambda (t) = \frac{f(t)}{1-F(t)}$, we know that $\Lambda (t) = - \log(1-F(t))$, where $f$ is the p.d.f. and $F$ is the c.d.f. of $T$. \\
It is easily seen that $\Lambda(T)$ has distribution $\mathrm{Exp}(1)$ (Exponential with intensity equal to~1): since $\Lambda$ is an increasing function, we have
\begin{align*}
\mathbb{P} (\Lambda (T) \ge t) &= \mathbb{P} (T \ge \Lambda^{-1}(t))
=\int_{\Lambda^{-1}(t)}^\infty f(s)ds \\
&= 1 - F(\Lambda^{-1}(t)) \\
&=\exp(-\Lambda(\Lambda^{-1}(t))) \\
&= \exp(-t),
\end{align*}
so that simulating failure times is simply achieved by using $T_i = \Lambda^{-1}(E_i)$ where $E_i \sim \mathrm{Exp}(1)$.
To compute $\Lambda$, we should have a parametric form for $\lambda_0$.
We assume that $T$ follows the Weibull distribution $\mathcal{W} (1,\nu)$ (when $x_i = 0$).
This choice is motivated by the following facts:
\begin{itemize}
\item Its cumulative hazard function is easy to invert. Indeed the hazard ratio is given by
$\lambda_0 (t) = \frac{\nu t^{\nu-1} e^{-t^\nu}}{1 - (1 - e^{-t^\nu})} = \nu t^{\nu - 1}$, so that $\Lambda^{-1} (y) = \left( \frac{y}{\exp(x_i^\top \theta)} \right)^{1 / \nu}$.

\item It enables two different trends - increasing or decreasing -- for the baseline hazard ratio that correspond to two typical behaviours in the medical field.
\begin{itemize}
\item decreasing: after taking a treatment, time before a side-effect's appearence
\item increasing: no memory process and patient's health is worsening
\end{itemize}
\end{itemize}

This method enables us to simulate $n$ failures times $T_1,T_2,\ldots,T_n$. \\
Then, we simulate $C_1, C_2,\ldots,C_n$ with exponential distribution. This finally gives us a set of observed times $(y_i)_{i=1}^n = (T_i \wedge C_i)_{i=1}^n$ and a set of censoring indicators $(\delta_i)_{i=1}^n = (\mathbbm{1}_{\{T_i \wedge C_i\}})_{i=1}^n$.


\section{Mini-batch sizing}
\label{appen:mb_size}
The mini-batch sizing question is essential since it is a natural trade-off between computing time and precision. We know that computing $\nabla f_i (\theta)$ needs the computation of $|R_i| \in \{1, \ldots, n_\text{pat} \}$ inner products.
One proves easily that computing a mini-batch $(1/n_\text{mb}) \nabla f_\mathcal{B} (\theta)$ - where $\mathcal{B}$ is the set of $n_\text{mb}$ index randomly picked - only needs $\max_{i \in \mathcal{B}} |R_i|$ inner products.
A simple probability  exercise gives us a key insight about the mini-batch size. \\
Let's assume that censoring is \emph{uniform} over the set $\{ 1, 2, \ldots, n_\text{pat} \}$ meaning that $| R_i | = c i$ with $c > 1$. Then, we denote $u_1, u_2, \ldots, u_{n_\text{mb}} \sim \mathcal{U}[n]$ the indices independently sampled to compute the mini-batch i.e. $\mathcal{B} = \{ u_i \}_{i=1}^{n_\text{mb}}$. Now we study the c.d.f. of $\max_{ 1 \le i \le n_\text{mb}} u_i$: for $k \in \{ 1, 2, \ldots, n \}$,
\begin{align*}
&\mathbb{P} \left( \max_{ 1 \le i \le n_\text{mb}} u_i \le k \right) = \prod_{i=1}^{n_\text{mb}} \mathbb{P}(u_i \le k) = \left( \frac{\floor*{k}}{n} \right)^{n_\text{mb}}, \\
&\mathbb{P} \left( \max_{ i \in \mathcal{B} } | R_{i} | \le c k \right) = \left( \frac{\floor*{k}}{n} \right)^{n_\text{mb}}, \\
&\mathbb{P} \left( \max_{ i \in \mathcal{B} } | R_{i} | \ge a \right) = 1 - \left( \frac{\floor*{a / c}}{n} \right)^{n_\text{mb}}, \mbox{ for } a < n_\text{pat}
\end{align*}
The third equation leads us to consider $1\ll n_\text{mb} \ll n$ to prevent both $ \max_{ i \in \mathcal{B} } | R_{i} |$ and $|\mathcal{B}|$ from being too large. This is why we used $n_\text{mb} = 0.1 n$ or $n_\text{mb} = 0.01 n$, depending of the size $n$ of the dataset. \\



\newpage

\bibliography{paper.bib}

\begin{thebibliography}{49}
\providecommand{\natexlab}[1]{#1}
\providecommand{\url}[1]{\texttt{#1}}
\expandafter\ifx\csname urlstyle\endcsname\relax
  \providecommand{\doi}[1]{doi: #1}\else
  \providecommand{\doi}{doi: \begingroup \urlstyle{rm}\Url}\fi

\bibitem[{Alizadeh} et~al.(2000){Alizadeh}, {Eisen}, {Davis}, {Ma}, {Lossos},
  {Rosenwald}, {Boldrick}, {Sabet}, {Tran}, {Yu}, et~al.]{alizadeh2000distinct}
A.~A. {Alizadeh}, M.~B. {Eisen}, R.~E. {Davis}, C.~{Ma}, I.~S. {Lossos},
  A.~{Rosenwald}, J.~C. {Boldrick}, H.~{Sabet}, T.~{Tran}, X.~{Yu}, et~al.
\newblock Distinct types of diffuse large b-cell lymphoma identified by gene
  expression profiling.
\newblock \emph{Nature}, 403\penalty0 (6769):\penalty0 503--511, 2000.

\bibitem[{Atchade} et~al.(2014){Atchade}, {Fort}, and {Moulines}]{moulines}
Y.~F. {Atchade}, G.~{Fort}, and E.~{Moulines}.
\newblock {On perturbed proximal gradient algorithms}.
\newblock \emph{ArXiv e-prints}, February 2014.

\bibitem[{Bach} and {Moulines}(2013)]{bach2013non}
F.~{Bach} and E.~{Moulines}.
\newblock Non-strongly-convex smooth stochastic approximation with convergence
  rate o (1/n).
\newblock In \emph{Advances in Neural Information Processing Systems}, pages
  773--781, 2013.

\bibitem[Beck and Teboulle(2009)]{FISTA}
A.~Beck and M.~Teboulle.
\newblock A fast iterative shrinkage-thresholding algorithm for linear inverse
  problems.
\newblock \emph{SIAM J. Img. Sci.}, 2\penalty0 (1):\penalty0 183--202, March
  2009.

\bibitem[{Boyd} and {Vandenberghe}(2004)]{boyd_book}
S.~{Boyd} and L.~{Vandenberghe}.
\newblock \emph{Convex Optimization}.
\newblock Cambridge University Press, New York, NY, USA, 2004.
\newblock ISBN 0521833787.

\bibitem[Carreira-Perpinan and Hinton(2005)]{carreira2005contrastive}
M.~Carreira-Perpinan and G.~Hinton.
\newblock On contrastive divergence learning.
\newblock In \emph{AISTATS}, volume~10, pages 33--40. Citeseer, 2005.

\bibitem[{Cox}(1972)]{CoxModel}
D.~R. {Cox}.
\newblock Regression models and life tables (with discussion).
\newblock \emph{Journal of the Royal Statistical Society, Series B},
  74:\penalty0 187--220, 1972.

\bibitem[{Dai} et~al.(2014){Dai}, {Xie}, {He}, {Liang}, {Raj}, {Balcan}, and
  {Song}]{double_sto}
B.~{Dai}, B.~{Xie}, N.~{He}, Y.~{Liang}, A.~{Raj}, M.~F. {Balcan}, and
  L.~{Song}.
\newblock Scalable kernel methods via doubly stochastic gradients.
\newblock \emph{CoRR}, abs/1407.5599, 2014.

\bibitem[{Defazio} et~al.(2014){Defazio}, {Bach}, and {Lacoste{-}Julien}]{SAGA}
A.~{Defazio}, F.~{Bach}, and S.~{Lacoste{-}Julien}.
\newblock {SAGA:} {A} fast incremental gradient method with support for
  non-strongly convex composite objectives.
\newblock \emph{CoRR}, abs/1407.0202, 2014.

\bibitem[{Duchi} et~al.(2011){Duchi}, {Hazan}, and {Singer}]{Duchi}
J.~{Duchi}, E.~{Hazan}, and Y.~{Singer}.
\newblock Adaptive subgradient methods for online learning and stochastic
  optimization.
\newblock \emph{J. Mach. Learn. Res.}, 12:\penalty0 2121--2159, July 2011.

\bibitem[Einav and Levin(2014)]{einav2014economics}
L.~Einav and J.~Levin.
\newblock Economics in the age of big data.
\newblock \emph{Science}, 346\penalty0 (6210):\penalty0 1243089, 2014.

\bibitem[{Fort} and {Moulines}(2003)]{fort2003}
G.~{Fort} and E.~{Moulines}.
\newblock Convergence of the monte carlo expectation maximization for curved
  exponential families.
\newblock \emph{Ann. Statist.}, 31\penalty0 (4):\penalty0 1220--1259, 08 2003.

\bibitem[Goeman(2010)]{goeman2010l1}
J.~Goeman.
\newblock L1 penalized estimation in the ox proportional hazards model.
\newblock \emph{Biometrical Journal}, 52\penalty0 (1):\penalty0 70--84, 2010.

\bibitem[Harikandeh et~al.(2015)Harikandeh, Ahmed, Virani, Schmidt,
  Kone{\v{c}}n{\`y}, and Sallinen]{harikandeh2015stopwasting}
R.~Harikandeh, M.~O. Ahmed, A.~Virani, M.~Schmidt, J.~Kone{\v{c}}n{\`y}, and
  S.~Sallinen.
\newblock Stopwasting my gradients: Practical svrg.
\newblock In \emph{Advances in Neural Information Processing Systems}, pages
  2251--2259, 2015.

\bibitem[Hinton(2002)]{hinton2002training}
G.~Hinton.
\newblock Training products of experts by minimizing contrastive divergence.
\newblock \emph{Neural computation}, 14\penalty0 (8):\penalty0 1771--1800,
  2002.

\bibitem[{Hu} et~al.(2009){Hu}, {Pan}, and {Kwok}]{hu_pan}
C.~{Hu}, W.~{Pan}, and J.~T. {Kwok}.
\newblock Accelerated gradient methods for stochastic optimization and online
  learning.
\newblock \emph{Advances in Neural Information Processing Systems}, 2009.

\bibitem[Johnson and Zhang(2013)]{SVRGold}
R.~Johnson and T.~Zhang.
\newblock Accelerating stochastic gradient descent using predictive variance
  reduction.
\newblock \emph{Advances in Neural Information Processing Systems 26}, pages
  315--323, 2013.

\bibitem[{Juditsky} and {Nemirovski}(2010)]{juditsky}
A.~{Juditsky} and A.~S. {Nemirovski}.
\newblock {First Order Methods for Nonsmooth Convex Large-Scale Optimization,
  II: Utilizing Problem's Structure}.
\newblock In \emph{{Optimization for Machine Learning}}, pages 29--63. {MIT
  Press}, August 2010.

\bibitem[{Kiefer} and {Wolfowitz}(1952)]{kiefer1952}
J.~{Kiefer} and J.~{Wolfowitz}.
\newblock Stochastic estimation of the maximum of a regression function.
\newblock \emph{Ann. Math. Statist.}, 23\penalty0 (3):\penalty0 462--466, 09
  1952.

\bibitem[{Konecny} et~al.(2016){Konecny}, {Liu}, {Richtarik}, and
  {Takac}]{semi-sgd}
J.~{Konecny}, J.~{Liu}, P.~{Richtarik}, and M.~{Takac}.
\newblock {Mini-Batch Semi-Stochastic Gradient Descent in the Proximal
  Setting}.
\newblock \emph{IEEE Journal of Selected Topics in Signal Processing},
  10:\penalty0 242--255, March 2016.

\bibitem[Lafferty et~al.(2001)Lafferty, McCallum, and Pereira]{crf1}
J.~D. Lafferty, A.~McCallum, and F.~C.~N. Pereira.
\newblock Conditional random fields: Probabilistic models for segmenting and
  labeling sequence data.
\newblock In \emph{Proceedings of the Eighteenth International Conference on
  Machine Learning}, pages 282--289, 2001.
\newblock ISBN 1-55860-778-1.

\bibitem[{Lan}(2010)]{Lan10}
G.~{Lan}.
\newblock An optimal method for stochastic composite optimization.
\newblock \emph{Mathematical Programming Series A}, 2010.

\bibitem[Lei and Jordan(2016)]{lei2016less}
L.~Lei and M.~I. Jordan.
\newblock Less than a single pass: Stochastically controlled stochastic
  gradient method.
\newblock \emph{arXiv preprint arXiv:1609.03261}, 2016.

\bibitem[Lin et~al.(2015)Lin, Mairal, and Harchaoui]{lin2015universal}
H.~Lin, J.~Mairal, and Z.~Harchaoui.
\newblock A universal catalyst for first-order optimization.
\newblock In \emph{Advances in Neural Information Processing Systems}, pages
  3384--3392, 2015.

\bibitem[Liu and Nocedal(1989)]{LBFGS}
D.~C. Liu and J.~Nocedal.
\newblock On the limited memory bfgs method for large scale optimization.
\newblock \emph{Math. Program.}, 45\penalty0 (3):\penalty0 503--528, December
  1989.

\bibitem[{Loi} et~al.(2007){Loi}, {Haibe-Kains}, {Desmedt}, {Lallemand},
  {Tutt}, {Gillet}, {Ellis}, {Harris}, {Bergh}, {Foekens},
  et~al.]{loi2007definition}
S.~{Loi}, B.~{Haibe-Kains}, C.~{Desmedt}, F.~{Lallemand}, A.~M. {Tutt},
  C.~{Gillet}, P.~{Ellis}, A.~{Harris}, J.~{Bergh}, J.~A. {Foekens}, et~al.
\newblock Definition of clinically distinct molecular subtypes in estrogen
  receptor--positive breast carcinomas through genomic grade.
\newblock \emph{Journal of clinical oncology}, 25\penalty0 (10):\penalty0
  1239--1246, 2007.

\bibitem[{Mittal} et~al.(2013){Mittal}, {Madigan}, {Cheng}, and
  {Burd}]{large-surv-analysis}
S.~{Mittal}, D.~{Madigan}, J.~{Cheng}, and R.~S. {Burd}.
\newblock Large-scale parametric survival analysis.
\newblock \emph{Statistics in medicine}, 32\penalty0 (23):\penalty0 3955--3971,
  10 2013.

\bibitem[Murdoch and Detsky(2013)]{murdoch2013inevitable}
T.~B. Murdoch and A.~S. Detsky.
\newblock The inevitable application of big data to health care.
\newblock \emph{Jama}, 309\penalty0 (13):\penalty0 1351--1352, 2013.

\bibitem[Murphy(2012)]{Murphy:2012:MLP:2380985}
K.~P. Murphy.
\newblock \emph{Machine Learning: A Probabilistic Perspective}.
\newblock The MIT Press, 2012.
\newblock ISBN 0262018020, 9780262018029.

\bibitem[{Nemirovski} et~al.(2009){Nemirovski}, {Juditsky}, {Lan}, and
  {Shapiro}]{nemirovski}
A.~{Nemirovski}, A.~{Juditsky}, G.~{Lan}, and A.~{Shapiro}.
\newblock Robust stochastic approximation approach to stochastic programming.
\newblock \emph{SIAM Journal on Optimization}, 19\penalty0 (4):\penalty0
  1574--1609, 2009.

\bibitem[Nitanda(2014)]{SVRG_MB}
A.~Nitanda.
\newblock Stochastic proximal gradient descent with acceleration techniques.
\newblock \emph{Advances in Neural Information Processing Systems}, pages
  1574--1582, 2014.

\bibitem[Park and Hastie(2007)]{park2007l1}
M.~Y. Park and T.~Hastie.
\newblock L1-regularization path algorithm for generalized linear models.
\newblock \emph{Journal of the Royal Statistical Society: Series B (Statistical
  Methodology)}, 69\penalty0 (4):\penalty0 659--677, 2007.

\bibitem[Richards(2012)]{richards2012handbook}
SJ~Richards.
\newblock A handbook of parametric survival models for actuarial use.
\newblock \emph{Scandinavian Actuarial Journal}, 2012\penalty0 (4):\penalty0
  233--257, 2012.

\bibitem[{Robbins} and {Monro}(1951)]{robbins1951}
H.~{Robbins} and S.~{Monro}.
\newblock A stochastic approximation method.
\newblock \emph{Ann. Math. Statist.}, 22\penalty0 (3):\penalty0 400--407, 09
  1951.

\bibitem[{Robert} and {Casella}(2004)]{robert2004monte}
C.~P. {Robert} and G.~{Casella}.
\newblock \emph{Monte Carlo statistical methods}, volume 319.
\newblock Springer New York, 2004.

\bibitem[{Schmidt} et~al.(2015){Schmidt}, {Babanezhad}, {Osama Ahmed},
  {Defazio}, {Clifton}, and {Sarkar}]{crf_schmidt2}
M.~{Schmidt}, R.~{Babanezhad}, M.~{Osama Ahmed}, A.~{Defazio}, A.~{Clifton},
  and A.~{Sarkar}.
\newblock {Non-Uniform Stochastic Average Gradient Method for Training
  Conditional Random Fields}.
\newblock pages 819--828, 2015.

\bibitem[{Schmidt} et~al.(2016){Schmidt}, {Le Roux}, and {Bach}]{SAG}
M.~{Schmidt}, N.~{Le Roux}, and F.~{Bach}.
\newblock {Minimizing Finite Sums with the Stochastic Average Gradient}.
\newblock \emph{Mathematical Programming}, pages 1--30, 2016.

\bibitem[{Shalev-Shwartz} and {Zhang}(2012)]{SDCA1}
S.~{Shalev-Shwartz} and T.~{Zhang}.
\newblock {Proximal Stochastic Dual Coordinate Ascent}.
\newblock \emph{ArXiv e-prints}, November 2012.

\bibitem[{Simon} et~al.(2011){Simon}, {Friedman}, {Hastie}, and
  {Tibshirani}]{coxnet}
N.~{Simon}, J.~{Friedman}, T.~{Hastie}, and R.~{Tibshirani}.
\newblock Regularization paths for cox's proportional hazards model via
  coordinate descent.
\newblock \emph{Journal of Statistical Software}, 39\penalty0 (5):\penalty0
  1--13, 3 2011.

\bibitem[Sohn et~al.(2009)Sohn, Kim, Jung, and Park]{sohn2009gradient}
I.~Sohn, J.~Kim, S.-H. Jung, and C.~Park.
\newblock Gradient lasso for cox proportional hazards model.
\newblock \emph{Bioinformatics}, 25\penalty0 (14):\penalty0 1775--1781, 2009.

\bibitem[{Tibshirani}(1997)]{CoxLasso}
R.~{Tibshirani}.
\newblock The lasso method for variable selection in the cox model.
\newblock \emph{Statistics in Medicine}, 16:\penalty0 385--395, 1997.

\bibitem[{Tibshirani} and {Hastie}(1990)]{GAD}
R.~{Tibshirani} and T.~{Hastie}.
\newblock \emph{Generalized Additive Models}, chapter~8.
\newblock Chapman and Hall, London, 1990.

\bibitem[{Toulis} and {Airoldi}(2014)]{implicit_sgd}
P.~{Toulis} and E.~M. {Airoldi}.
\newblock {Implicit stochastic gradient descent for principled estimation with
  large datasets}.
\newblock \emph{ArXiv e-prints}, August 2014.

\bibitem[{Van De Vijver} et~al.(2002){Van De Vijver}, {He}, {van't Veer},
  {Dai}, {Hart}, {Voskuil}, {Schreiber}, {Peterse}, {Roberts}, {Marton},
  et~al.]{van2002gene}
M.~J. {Van De Vijver}, Y.~D. {He}, L.~J. {van't Veer}, H.~{Dai}, A.~A.~M.
  {Hart}, D.~W. {Voskuil}, G.~J. {Schreiber}, J.~L. {Peterse}, C.~{Roberts},
  M.~J. {Marton}, et~al.
\newblock A gene-expression signature as a predictor of survival in breast
  cancer.
\newblock \emph{New England Journal of Medicine}, 347\penalty0 (25):\penalty0
  1999--2009, 2002.

\bibitem[Witten and Tibshirani(2009)]{witten2009survival}
D.~M. Witten and R.~Tibshirani.
\newblock Survival analysis with high-dimensional covariates.
\newblock \emph{Statistical methods in medical research}, 2009.

\bibitem[{Xiao} and {Zhang}(2014)]{SVRG}
L.~{Xiao} and T.~{Zhang}.
\newblock \emph{SIAM Journal on Optimization}, 24\penalty0 (4):\penalty0
  2057--2075, 2014.

\bibitem[{Xiao}(2009)]{lin_xiao}
X.~{Xiao}.
\newblock Dual averaging methods for regularized stochastic learning and online
  optimization.
\newblock In \emph{In Advances in Neural Information Processing Systems 23},
  2009.

\bibitem[{Yang} and {Zhou}(2013)]{cocktail}
Y.~{Yang} and H.~{Zhou}.
\newblock A cocktail algorithm for solving the elastic net penalized cox's
  regression in high dimensions.
\newblock \emph{Statistics and Its Interface}, 6:\penalty0 167--173, 2013.

\bibitem[Zhang and Oles(2000)]{zhang2000value}
T.~Zhang and F.~Oles.
\newblock The value of unlabeled data for classification problems.
\newblock In \emph{Proceedings of the Seventeenth International Conference on
  Machine Learning,(Langley, P., ed.)}, pages 1191--1198. Citeseer, 2000.

\end{thebibliography}

\bibliographystyle{plainnat}

\clearpage

\end{document}